\documentclass[a4paper]{article}
\usepackage[english]{babel}

\usepackage{authblk}
\usepackage{natbib}
\usepackage{graphicx}
\usepackage{bbold}
\usepackage{a4wide}
\usepackage{verbatim}
\usepackage{color}
\usepackage{dsfont}

\usepackage{amsmath}
\usepackage{amssymb}
\usepackage{amsfonts}
\usepackage{amsthm}
\usepackage{xurl}

\usepackage{algorithm}
\usepackage{algorithmic}

\def\signed #1{{\leavevmode\unskip\nobreak\hfil\penalty50\hskip2em
  \hbox{}\nobreak\hfil(#1)%
  \parfillskip=0pt \finalhyphendemerits=0 \endgraf}}

\newsavebox\mybox
\newenvironment{aquote}[1]
  {\savebox\mybox{#1}\begin{quote}}
  {\signed{\usebox\mybox}\end{quote}}

\newtheorem*{theorem*}{Theorem}
\newtheorem{theorem}{Theorem}
\newtheorem{definition}{Definition}

\newcommand\intMax{\mbox{\scriptsize{INT\textunderscore MAX}}}
\newcommand\longIntMax{\mbox{\scriptsize{LONG\textunderscore INT\textunderscore MAX}}}
\newcommand\doubleEpsilon{\mbox{\scriptsize{DBL\textunderscore EPSILON}}}
\newcommand\doubleMin{\mbox{\scriptsize{DBL\textunderscore MIN}}}
\newcommand\doubleMax{\mbox{\scriptsize{DBL\textunderscore MAX}}}

\begin{document}
\title{Two-level histograms for dealing with outliers\\and heavy tail distributions}
\author{Marc Boull\'e}
\affil{Orange Labs - 22300 Lannion - France}
\maketitle

\abstract{
Histograms are among the most popular methods used in exploratory analysis to summarize univariate distributions.
In particular, irregular histograms are good non-parametric density estimators that require very few parameters: the number of bins with their lengths and frequencies.
Many approaches have been proposed in the literature to infer these parameters, either assuming hypotheses about the underlying data distributions or exploiting a model selection approach.
In this paper, we focus on the G-Enum histogram method, which exploits the Minimum Description Length (MDL) principle to build histograms without any user parameter and achieves state-of-the art performance w.r.t accuracy; parsimony and computation time.
We investigate on the limits of this method in the case of outliers or heavy-tailed distributions.
We suggest a two-level heuristic to deal with such cases. The first level exploits a logarithmic transformation of the data to split the data set into a list of data subsets with a controlled range of values. The second level builds a sub-histogram for each data subset and aggregates them to obtain a complete histogram.
Extensive experiments show the benefits of the approach.
}

\section{Introduction}
\label{sec:introduction}

Histograms are among the most popular methods used in exploratory analysis to summarize univariate distributions.
\emph{Regular histograms} are the simplest savor of histograms to represent a distribution: all bins are of the same width and the only parameter to select is the number of bins. While they are suited to roughly uniform distributions \cite{Rissanen1992densityestimation}, they fail to capture the density of more complex distributions.
\emph{Irregular histograms} are non-parametric piecewise constant density estimators that require very few parameters: the number of bins with their widths and frequencies.
Several irregular histogram methods have been proposed in the literature, but they often require user-defined parameters, such as the number of bins or the accuracy $\epsilon$ at which the data is
to be approximated. For example, the minimum description length (MDL) histogram methods \cite{Rissanen1992densityestimation,pmlrv2kontkanen07a} automatically choose the number of bins and their widths, but these widths need to be a multiples of a $\epsilon$ user parameter.
In the context of exploratory analysis, the choice of this parameter is not an easy task, and fully automatic histogram methods are preferable.
Several automatic irregular histogram methods have been proposed in the literature, such as the \emph{taut string} methods based on penalized likelihood \cite{Davies2004,RMG2010}, the \emph{Bayesian blocks} histograms based Bayesian regularization \cite{Scargle2013} or the \emph{G-Enum} method \cite{ZelayaEtAl23} based on the MDL approach.
In a comparison between several regular and irregular histograms methods, the G-Enum method achieves state-of-the-art accuracy for estimated density while being much more scalable than its closest competitors \cite{ZelayaEtAl23}. It is also among the most parsimonious methods, with far fewer intervals than the most accurate alternative methods, which is an essential feature for exploratory analysis when interpretability is an issue.
These properties being in line with our main objective in this paper, we focus on this method.

The G-Enum method extends the MDL method \cite{pmlrv2kontkanen07a} with an automatic choice of $\epsilon$, a fast to compute closed-form evaluation criterion and scalable efficient optimization heuristics.
Its modeling space is described on the basis of $\epsilon$-length elementary bins, where each histogram bin consists of a subset of adjacent $\epsilon$-length  bins.
A granularity parameter is exploited to automatically select the $\epsilon$ parameter.
Together with efficient linearithmic optimization heuristic, this granulated MDL criterion provides a resilient, efficient and fully automated approach to histogram density estimation.
Nevertheless, this method reaches its limits in the case of outliers or heavy-tailed distributions.
We suggest a two-level heuristic to deal with such cases. The first level exploits a logarithmic transformation of the data to split the data set into a list of data subsets with a controlled range of values. The second level builds a sub-histogram for each data subset and aggregates them to obtain a complete histogram.

The rest of the paper is organized as follows.
We briefly recall the G-Enum method in Section~\ref{sec:G-Enum}.
We illustrate the limit of histogram methods in the case of outliers and discuss possible solutions to push these limits in Section~\ref{sec:outliersLimits}.
We suggest a two-level approach for building histograms in Section~\ref{sec:methodDescription},
 and analyze its properties in Section~\ref{sec:analysis}
We perform extensive experiments with artificial data sets in Section~\ref{sec:evaluation}.
Finally, we suggest future work in Section~\ref{sec:futureWork} and give a summary in Section~\ref{sec:conclusion}.

\section{G-Enum method: summary}
\label{sec:G-Enum}

This section is a brief reminder of the G-Enum method \cite{ZelayaEtAl23}.

\subsection{Problem formulation}
\label{sec:formulation}
We consider a sample of $n$ observations $x^n = (x_1,...,x_n)$ on the interval $[x_{min}, x_{max}]$. Let $\epsilon$ be the approximation accuracy, so that each $x_j \in x^n$  can be approximated by $\widetilde{x}_j \in\mathcal{X}=\{ x_{min} + t\epsilon ; t = 0,... , E\}$ where $\displaystyle E =  L/\epsilon$ and  $L = x_{max} - x_{min}$ is the `{\em domain length}' of the data. We expect to have $E \in \mathbb{N}$.\\

Let $\mathcal{C}$ be the set of possible endpoints for sub-intervals as 
\begin{equation*}
\mathcal{C} = \{c_t=x_{min} - \epsilon/2 + t\epsilon ; t = 0,\ldots , E \}
\end{equation*}

These endpoints define $E$ {\em elementary bins}  of length $\epsilon$, which are called $\epsilon$-bins. They are the building blocks of histogram intervals: each combination of $\epsilon$-bins into $K$ intervals, with $K$ ranging from 1 to $E$, defines a histogram model. In this range of possibilities, the goal is to select a set of $K-1$ endpoint $C=(c_1,...,c_{K-1}),~ c_k \in \mathcal{C}$ such that $[c_0, c_E]=[x_{min} - \epsilon/2, x_{max} + \epsilon/2]$ is partitioned into $K$ intervals $\{ [c_0, c_1],]c_1, c_2], ..., ]c_{K-1},c_E]\}$ that are well-suited to the actual data distribution. Each interval $k$ has a data count of $h_k$ entries and a length $L_k=c_k - c_{k-1}$, which is a multiple of $\epsilon$: \begin{equation*}
\forall k, ~\exists ~ E_k \in \mathbb{N} \textrm{~ such that~} L_k = E_k \cdot \epsilon
\end{equation*}

A histogram model is entirely defined by the choice of the number of intervals, the set of endpoints that define them and their data counts. We thus note a histogram model $\mathcal{M} = (K, C, \{h_k\}_{1 \leq k \leq K}) $. 
The relevance of each model can be measured through different types of MDL criteria, for example using an enumerative criterion.

\subsubsection{Granularity and choice of $\epsilon$}

To get rid of the user parameter $\epsilon$, a new method parameter is introduced, that will automatically be inferred.
Let $G$ be the granularity parameter.
For a given $E$, the numerical domain is split into $G$ bins ($1 \leq G \leq E$) of equal width.
In practice, the constant $E=10^9$ is used, which is both close to the limits of the representation of machine integers and allows to obtain very accurate histograms, with an accuracy of up to one billionth of the value domain.
Each of these new elementary bins, that are called {\em $g$-bins}, is composed of $g = E/G$ $\epsilon$-bins. Each of the intervals of any histogram constructed has then a length that is a multiple of these $g$-bins. In other words, each interval is no longer composed of a multiple of $\epsilon$-bins but rather composed of $G_k$ $g$-bins. 

This new criterion, which is called {\bf G-Enum} is still very similar to the MDL-based enumerative criterion {\bf Enum} for histograms, as shown in table \ref{tab:comp-GMODL}.

\subsection{Enum and G-Enum criteria for histogram models}

\renewcommand{\arraystretch}{2}
\begin{table*}[!ht]
\centering
\caption{Term comparison of the Enum and G-Enum criteria}
\begin{tabular}{|p{1.2cm}|p{3cm}|p{4.8cm}|p{2.6cm}|}
\hline
Criterion & Indexing terms & Multinomial terms &  Bin index terms\\\hline\hline
Enum &  $ \displaystyle \log{}^* K + \log{} \binom{E+K-1}{K-1}$ & $\displaystyle \log{} \binom{n+K-1}{K-1}+  \log{} \frac{n!}{h_1!... h_K!}$ & $\sum^K_{k=1} h_k \log{} E_k$\\ \hline
\mbox{G-Enum} &  $ \displaystyle  \log{}^* K +\log^* G + \log{} {{G+K-1}\choose{K-1}}$ & $\displaystyle  \log{} \binom{n+K-1}{K-1}+  \log{} \frac{n!}{h_1!... h_K!}$ &$\sum^K_{k=1} h_k \log{} G_k + n \log{} \frac{E}{G}$\\\hline
\end{tabular}
\label{tab:comp-GMODL}
\end{table*}

Table~\ref{tab:comp-GMODL} recalls the Enum criterion for histogram models and its granulated extension G-Enum.
The $\log{}^*K$ and $\log{}^*G$ prior terms encode the choice of the number of intervals and of the granularity parameter. They exploit Rissanen's universal prior for integers \cite{rissanen1983}, that favors small integers, i.e. simpler histograms.
The $\log{} \binom{G+K-1}{K-1}$ term encodes the boundaries of the intervals at the granularity precision.
The multinomial terms are used to encode the multinomial distribution of the $n$ instances on the $K$ intervals.
They rely on an enumerative criterion with appealing optimality properties \cite{BoulleEtAlArxiv16}.
The $\sum^K_{k=1} h_k \log{} G_k + n \log{} \frac{E}{G}$ term encodes the position of the $h_k$ instances of each interval on the $E_k = G_k \frac{E}{G}$ elementary $\epsilon$-bins of the interval.

\subsection{Optimization algorithms}

For additive criteria such as Enum, a dynamic programming algorithm can be applied to obtain the optimal solution. However, its computational complexity is cubic w.r.t. the size of the data, which makes it impractical  in the case of large data sets.
The G-Enum method exploits a greedy bottom-up optimization heuristic followed by post-optimization steps that mainly consist in adding, removing, or moving endpoints around the locally optimal solution.
Experiments in \cite{ZelayaEtAl23} show that the accuracy of histograms optimized using these heuristics is indistinguishable from those using the optimal algorithm, while the computational complexity is O($n \log n$) instead of O$(n^3)$.

\subsection{Experimental results}

We summarize below the results of the comparative experiments performed to evaluate the G-Enum method \cite{ZelayaEtAl23}.
The comparison include the following irregular and regular histogram methods:
\begin{itemize}
\item \texttt{G-Enum}, the method summarized in this section,
\item \texttt{Taut string} histograms \cite{Davies2004,Davies2009},
\item \texttt{RMG} histograms \cite{RMG2010},
\item \texttt{Bayesian blocks} \cite{Scargle2013},
\item Sturges rule histograms,
\item \texttt{Freedman-Diaconis} rule histograms \cite{Freedman1981}.
\end{itemize}

They are evaluated on artificial datasets with know distributions: Normal, Cauchy, Uniform, Triangle, Triangle mixture and Gaussian mixture.
The methods are compared on three criterions: parsimony using the number of intervals, accuracy evaluated with the Hellinger distance and computation time.
The analysis of the experimental results show that the \texttt{G-Enum} method achieves state of the art accuracy while being much more parsimonious and fast its closest competitors.
\begin{aquote} {\cite{ZelayaEtAl23}}
"Although rarely the best for each distribution type, \texttt{G-Enum} histograms are consistently among the best estimators, and this without the high variability of the other methods. Focusing on irregular histograms, \texttt{G-Enum} is certainly among the most parsimonious in number of intervals. For exploratory analysis, this is an important quality because it makes the interpretation of the results easier and more reliable. \texttt{G-Enum} is also by far the fastest of irregular methods, making it suitable to large data sets."
\end{aquote}

\section{Limits of histogram methods w.r.t. outliers}
\label{sec:outliersLimits}

We first give an illustrative example of the limits of the G-Enum method in the case of outliers, and then discuss possible solutions to push these limits.

\subsection{Illustative exemple}
\label{sec:outliersLimitsExample}
Let us consider a data set containing $n=10,000$ data entries distributed according to a Gaussian distribution $G(\mu=0, \sigma=1)$. The range of the numerical domain is $L = (x_{max}-x_{min})$.
As $\sigma=1$, we have $L \leq 10$ with high probability. 
The range of the numerical domain at $\epsilon$ accuracy is $E=L/\epsilon$.
Let us recall that we have chosen $E=10^9$ to be compliant with the computer representation of integers using four bytes. As a matter of fact, computer integers are in the value domain $]-\intMax; \intMax[$, with $\intMax=2^{31}\approx 2.10^9$.
Using the $E=10^9$ precision parameter, the bounds of the histogram intervals are very precise, and the underlying distribution can be very well approximated as the number of data entries $n$ increases.

Let us now assume that we have an outlier data entry in our data set, with value $x_{out}=10^{12}$. The range of the value domain becomes $L\approx 10^{12}$ and using the same precision parameter $E=10^9$ amounts to setting $\epsilon \approx 1000$. With this $\epsilon$ parameter, the optimal histogram reduces to a histogram with two intervals, consisting of a first interval of width $E_1=1$ that contains all the $n$ initial Gaussian data entries in a bin of width 1000, and a second interval of width $E_2=E-1$ containing the outlier data entry. The quality of the histogram becomes very poor as the whole data set except one outlier is summarized using one single interval.

Let us note that, to the best of our knowledge, this problem is likely to occur with most alternative histogram methods.
In the following we investigate on solutions to push these limits.

\subsection{Possible solutions to push the limits of the method}
We suggest three possible solutions to push the limits of the method and summarize their potential benefits and drawbacks.

\subsubsection{Use of long integers}
\label{sec:outliersWithLongInts}

One computer-based solution consists in using long integers instead of standard integers for the choice of our precision parameter $E$.
We could then extend the precision parameter to $E=10^{18}$ and be compliant with the computer representation of long integers using eight bytes, in the value domain $]-\longIntMax; \longIntMax[$, where $\longIntMax=2^{63}\approx 9.10^{18}$.

Unfortunately, this solution is not likely to work well.
First, it extends the outlier limits by "`only"' nine additional orders of magnitude.
Second, this long int based choice of $E$ raises critical numerical issues in the optimization algorithm.

For example, let us assume that we have an interval $i$ with length $E_i \gg 1$ and frequency $h_i$. Let us consider the merge of this interval with a singleton interval $j$ of width $E_j=1$ and frequency $h_j=1$.
The likelihood part $C_w()$ of the histogram cost criterion related to the width of the intervals is
\begin{eqnarray*}
C_w(i) &=& h_i \log E_i,\\
C_w(j) &=& 0,\\
C_w(i \cup j) &=& h_i \log (E_i+1).
\end{eqnarray*}

The variation of cost $\delta C_w$ is then
\begin{eqnarray*}
\delta C_w &=& C_w(i \cup j)-C_w(i)-C_w(j),\\
  &=& h_i (\log (E_i+1) - \log(E_i)),\\
  &=& h_i (\log E_i(1+1/E_i) - \log(E_i)),\\
  &=& h_i \log (1+1/E_i),\\
	&\approx& h_i/E_i.
\end{eqnarray*}

On a computer, real values are stored using a floating-point representation with a mantissa up to 15 digits ($\doubleEpsilon \approx 2. 10^{-16}$). Two distinct values will be equal if their relative difference is lower than $\doubleEpsilon$.
Back to our optimization algorithm, for $h_i\approx 1$ and $E_i \approx E$, we get $\delta C_w \approx 10^{-18}=0$. Therefore, finding the best merge of intervals may be impossible in some tricky cases.

\subsubsection{Extension to hierarchical histogram models}
\label{sec:outliersWithHierarchicalModels}

One solution to cope with outliers consists in extending the G-Enum method to a hierarchical model. 
A histogram consists in a set of adjacent intervals, whereas a hierarchical histogram consists in a tree of intervals, where:
\begin{itemize}
	\item  each leaf node is an interval,
	\item each intermediate node can be seen both as an interval, union of its children intervals, and as a histogram, set of its children intervals,
	\item the root node represents the whole value domain.
\end{itemize}
Such a hierarchical histogram could potentially cope with outliers.
For example, using the data set described in Section~\ref{sec:outliersLimitsExample}, we could have one root node with three children nodes; the first one for all the Gaussian data entries, the second one with an empty interval and the last one with the outlier. Then the first node could be divided again so as to produce a standard histogram focused on the Gaussian data entries, without any outlier issue.

This possible solution looks appealing, but its implementation may encounter several problems:
\begin{itemize}
	\item devising an effective prior for hierarchical models is not an easy task,
	\item optimizing hierarchical models is known to be difficult, with little hope of achieving optimality efficiently,
	\item the optimization algorithm may face numerical problems, since many models to be compared may have almost the same cost.
\end{itemize}

\subsubsection{Exploitation of the properties of floating-point representation}
\label{sec:outliersWithFloatingPoint}

Let us first summarize how real values are encoded on computers using a floating-point representation. Computer real values are stored on 8 bytes and thus encoded using 64 bits:
\begin{itemize}
	\item 1 bit for the sign: -1 or +1,
	\item 11 bits for the exponent: between $\doubleMin=10^{-308}$ and $\doubleMax=10^{308}$,
	\item 52 bits for the mantissa: about 15 digits, for mantissa in interval $[1;10[$.
\end{itemize}

Whereas mathematical real values that belong to $\mathbb{R}$ are continuous and unbounded, computer real values are discrete in essence and bounded. They belong to a finite set $\mathbb{R}^{(cr)}$ (where $(cr)$ stand for \emph{computer representation}). The set $\mathbb{R}^{(cr)}$ contains $2^{64} ~ \approx 1.8.10^{19}$ distinct values that belong to
the finite numerical domain $[-10^{308};-10^{-308}] \cup \{0\} \cup [10^{-308};10^{308}]$.
Let us note that all computer real values have an approximately constant relative precision related to the mantissa, but an absolute precision that exponentially increases around the value 0. There are more than 600 orders of magnitude of difference of absolute precision between the largest and the smallest computer real values. In other terms, mathematical real values have translation-invariant
density properties all over $\mathbb{R}$ (like in the case of fixed-point representation values). Conversely, the density of floating-point representation values in $\mathbb{R}^{(cr)}$ is heavily peaked around the value 0: it increases exponentially for $x \rightarrow 0$ until reaching the underflow regime and decrease exponentially for $x \rightarrow \infty$ until reaching the overflow regime.

Histograms where the width of intervals are multiple of $\epsilon$-bins rely on a constant absolute precision and they cannot cope well with outliers.
We suggest to investigate the properties of floating-point representation to extend the G-Enum method. This is detailed in next section.

\section{Two-level method for histograms}
\label{sec:methodDescription}

The principle of the method is to build a histogram directly from a data set  only if the result is likely to be of  sufficient quality. Otherwise, the data set is split into data subsets and a global histogram is obtained by aggregating the sub histograms built from each data subset.
Note that contrary to the hierarchical models suggested in Section~\ref{sec:outliersWithHierarchicalModels}, the method outputs a single global histogram, not a hierarchy of histogram.

In this section, we first introduce a quality criterion based on the notion of well conditioned data set for histograms.
We then present a log-transformation method that can be applied to any data set and will be used to  effectively split the data set into data subsets.
We also suggest a way to get around the limits of floating-point representation.
We finally detail the two-level method that exploits the quality criterion and the split heuristic.

\subsection{Well conditioned data sets for histograms}
\label{sec:wchProperty}

Let introduce the notion of well conditioned data sets for histograms.

\begin{definition}
A data set $\mathcal{D}$ is \emph{well conditioned for histograms} (WCH) of $\epsilon$-bin length $E$ if all its data entries with distinct values can be separated in different intervals. Otherwise, a data set is said \emph{ill conditioned for histograms} (ICH). 
\end{definition}

If a data set is well conditioned, histograms can be build without any risk of loss of numerical precision.
To investigate this notion, let us first define some characteristics of data sets.

\begin{definition}
The \emph{range} of a data set $\mathcal{D}$ is defined as $rng(\mathcal{D}) = \max_{\mathcal{D}} x - \min_{\mathcal{D}} x$, that is the difference between it maximum and minimum values.
\end{definition}

\begin{definition}
The \emph{precision} of a data set $\mathcal{D}$ is defined as $pr(\mathcal{D}) = \min_{\mathcal{D}, \delta x>0} \delta x$, that if the min difference between two successive distinct values.
\end{definition}

\begin{definition}
The \emph{granular length} of a data set $\mathcal{D}$ is defined as $gr(\mathcal{D}) = rng(\mathcal{D})/pr(\mathcal{D})$. 
\end{definition}

The following results are trivial and given without proof.

\begin{theorem}
A data set $\mathcal{D}$ is ill conditioned for histograms if its precision is smaller that the $\epsilon$-bin length of the histogram, or if its granular length is larger that number $E$ of $\epsilon$-bins.
More formally, we have:
\begin{itemize}
	\item $\mathcal{D}$ is ICH $\Leftrightarrow pr(\mathcal{D}) < \epsilon$,
	\item $\mathcal{D}$ is ICH $\Leftrightarrow gr(\mathcal{D}) > E$.
\end{itemize}
\end{theorem}

\begin{theorem}
The WCH (resp. ICH) property of a data set $\mathcal{D} \subset \mathbb{R}$ is invariant under any linear transformation of the data entries of $\mathcal{D}$.
\end{theorem}

Let us now define the notion of \emph{histogram collision} in a data set $\mathcal{D}$ as the case where two data entries with distinct values fall in the same $\epsilon$-bin of a histogram.
It is noteworthy that the focus is on being able to separate data entries with distinct values, not to separate any data entries that may share the same value.
A data set is ill conditioned for histograms if its number of collisions is greater or equal than 1.
Evaluating the risk of loss of precision while building a histogram from a data set relates to evaluating its ICH property. This can be done in $O(n \log n)$ either by computing the range and precision of the data set or alternatively by counting its number of collisions.


We can notice that the range and precision of a data set are characteristics that are related to extreme value statistics and that they are likely to exhibit a very large variance.
Inspired by robust statistics, we suggest a stronger condition for the ICH property, with a threshold $t_c$ for a minimum number of collisions.
Having $t_c$ collisions in the data set means that $t_c$ data entries fall in a set of colliding bins, each one containing at least two data entries with distinct values.
This covers two extreme cases: a flat one where each colliding bin contains only two data entries and a peaked one with one single colliding bin containing $t_c$ data entries.
We choose to exploit a condition for the ICH property based on the peaked case with $1 < t_c \ll n$ because it is likely to require larger data sets to trigger the condition while minimizing the potential loss of accuracy w.r.t. interval bounds.

Let us introduced another threshold $t_E$ as the number of $\epsilon$-bins in a histogram used to evaluate the ICH property. Although $t_E=E$ seems a natural choice, let us recall that the choice $E=10^9$ is not driven by a required accuracy of one billionths. In fact, the G-Enum method optimizes the granularity of histograms that rely on $G$ bins, $1 \leq G \leq E$, and convergence is expected as $E \rightarrow \infty$. The value $E=10^9$ was then chosen to be as large as possible within the computer numerical limits. We hope that the optimal granularity can be found for $G \ll E$ to avoid potential instabilities around the point of convergence.

In the end, we choose the thresholds $t_c=\log{n}, t_E = \sqrt{E} \log{E}$ and introduce the \emph{PICH} criterion in Definition~\ref{def:PICH}.

\begin{definition}
\label{def:PICH}
A data set $\mathcal{D}$ of size $n$ is \emph{practically ill conditioned for histograms} (PICH) built upon $E$ elementary $\epsilon$-bins if at least one colliding bin within a granularized histogram with $G=\sqrt{E}\log{E}$ bins contains more than $\log{n}$ data entries. Otherwise, the data set is \emph{practically well conditioned for histograms} (PWCH)
\end{definition}

This heuristic PICH criterion is designed to push the limits of the method's applicability.
In practice, the thresholds $t_c$ and $t_E$ have been chosen to jointly optimize a set of competing criteria, which are summarized below.

\begin{itemize}
  \item \emph{automation}
	\begin{itemize}
		\item a parameter-less criterion is important so that data scientists can actually spend more time on the business problem at hand,
	\end{itemize}

  \item \emph{theoretical optimality}
	\begin{itemize}
		\item although $E \rightarrow \infty$ should be considered, $E$ is set to $10^9$ which is close to the computer numerical limits,
		\item $t_E$ is as small as possible to avoid potential instability during the optimization of the granularity in the G-Enum method,
	\end{itemize}

  \item \emph{accuracy}
	\begin{itemize}
		\item $t_c$ is as small as possible to minimize to potential loss of accuracy for the interval bounds,
		\item $t_E$ is as large as possible to allow accurate granularities,
	\end{itemize}

  \item \emph{scalability}
	\begin{itemize}
		\item using large enough $t_c$ and $t_E$ thresholds, the PICH criterion should be conservative enough to avoid triggering advanced heuristics too often in the case of ill conditioned data sets,
		\item although optimal algorithms look appealing, only heuristics with at most super-linear time complexity can be used in the case of large real world data sets.
	\end{itemize}
\end{itemize}

The ICH property and the PICH criterion are further investigated in Section~\ref{sec:evaluation}, with a sensitivity analysis w.r.t. the $t_E$ threshold.

\subsection{Log-transformation of computer real numbers}
\label{sec:logTransformation}

Let us first introduce a new function $\log^{(cr)}$, that extends the standard $\log$ function to any negative, null or positive computer real value:
\begin{eqnarray*}
\log^{(cr)}(x) &=& -\doubleEpsilon-(\log -x-\log \doubleMin), \forall x \in \mathbb{R}_{-}^{*(cr)},\\
\log^{(cr)}(0) &=& 0,\\
\log^{(cr)}(x) &=& \doubleEpsilon+\log x - \log \doubleMin, \forall x \in \mathbb{R}_{+}^{*(cr)}.
\end{eqnarray*}

For $x=mant \times 10^{exp}$, we have $\log x= \log (mant) + exp \times \log 10$.
We now evaluate the bounds of the set $\log(\mathbb{R}^{(cr)})$ obtained after the log-transformation of $\mathbb{R}^{(cr)}$:
\begin{eqnarray*}
\sup(\log(\mathbb{R}^{(cr)})) &=& \log(\doubleEpsilon) + \log \doubleMax - \log \doubleMin,\\
  &\approx& \log 10^{-15} + \log 10^{308} - \log 10^{-308},\\
	&\approx& 600 \times \log 10,\\
\inf(\log(\mathbb{R}^{(cr)})) &=& - \sup(\log(\mathbb{R}^{(cr)})).
\end{eqnarray*}

Whereas the values of $\log(\mathbb{R}^{(cr)})$ exploit the mantissa with the same limits as in $\mathbb{R}^{(cr)}$, the exponents in $\log(\mathbb{R}^{(cr)})$ are bounded by around 3, that is about one hundredth of the related bound in $\mathbb{R}^{(cr)}$ (since $308 ~\approx 3 \times 100$).
The set $\log(\mathbb{R}^{(cr)})$ thus contains about $10^{17}$ distincts values, approximately 100 times less than its super set $\mathbb{R}^{(cr)}$.
Conversely, the log values have approximately the same absolute precision (15 digits) and are almost uniformly distributed on the numerical domain $[-600 \times \log 10;600 \times \log 10]$.

To summarize, this log-transformation provides a monotonous transformation of the \emph{initial values} in $\mathbb{R}^{(cr)}$ to \emph{log values} in $\log (\mathbb{R}^{(cr)})$, with an almost constant density on a smaller value domain.
Despite the decrease of size compared to the $\mathbb{R}^{(cr)}$, we suggest that these properties of $\log (\mathbb{R}^{(cr)})$ are particularly suitable for histograms, which assume a piecewise constant density per interval. 
Let us notice that for a histogram build on $\mathbb{R}^{(cr)}$, each $\epsilon$-bin represents a sub set with a constant maximum \emph{absolute difference} of values. Conversely, on $\log(\mathbb{R}^{(cr)})$, each $\epsilon$-bin represents a sub set with a constant maximum \emph{relative difference} of values.
We thus expect the $\log(\mathbb{R}^{(cr)})$ space to be well suited for dividing a data set with a wide range of values into data subsets with limited range of values.

\smallskip

\begin{figure}[htbp!]
\begin{center}
\includegraphics[width=0.7\columnwidth]{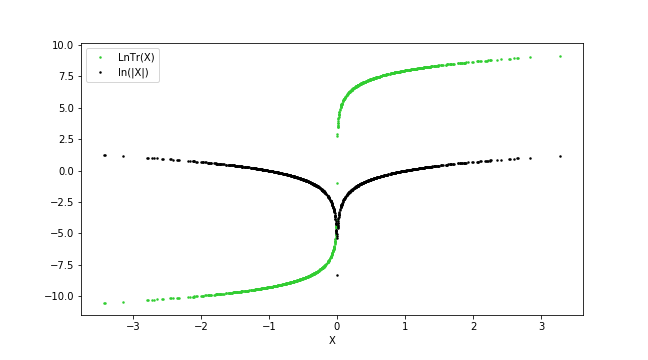}
\end{center}
\caption{Log-transformation of data generated from a Gaussian distribution $G(\mu=0, \sigma=1)$.}
\label{fig:LogTrG_0_1}
\end{figure}

In the case of a data set $\mathcal{D}$ to analyze, we suggest to adapt the $\log^{(cr)}$ function in order to reduce the range of values and to avoid the potential gaps around the value 0:
\begin{eqnarray*}
\log_{\mathcal{D}}^{(cr)}(x) &=& - \min_{\mathcal{D_{-}^{*}},\delta x > 0} {\delta \log x} -(\log -x-\log \min_\mathcal{D_{-}^{*}} -x), \forall x \in \mathcal{D}_{-}^{*},\\
\log_{\mathcal{D}}^{(cr)}(0) &=& 0,\\
\log_{\mathcal{D}}^{(cr)}(x) &=& \min_{\mathcal{D_{+}^{*}},\delta x > 0} {\delta \log x} +\log x - \log \min_\mathcal{D_{+}^{*}} x, \forall x \in \mathcal{D}_{+}^{*}.
\end{eqnarray*}

This is illustrated in Figure~\ref{fig:LogTrG_0_1} in the the case of a data set of size $n=1000$ drawn from a Gaussian distribution $G(\mu=0, \sigma=1)$.
Mainly, the log-transformation exploits the opposite of the function $\log -x$ for the negative values and the function $\log x$ for the positive values, and shifts them to achieve a smooth monotonous transformation of all the values.

\subsection{Dealing with the limits of floating-point representation}
\label{sec:mantissaLimits}

Whereas the PWCH criterion allows to cope with data sets with very large range of values, new numerical limits can be encountered for data sets with very small range.

As an example, let us take a data set $\mathcal{D}$ with a range of 1, $rng(\mathcal{D}) = \max_{\mathcal{D}} x - \min_{\mathcal{D}} x=1$.
Let us assume that this data set is PWCH, so that the G-Enum method is able to separate all its data entries using $E=10^9$ $\epsilon$-bins.
If $\min_{\mathcal{D}} x=1$ and $\max_{\mathcal{D}} x=2$, the 15 digits of the mantissa of computer real values (cf. Section~\ref{sec:outliersWithFloatingPoint}) allow to encode the boundaries of the $\epsilon$-bins with an excellent precision.
If $\min_{\mathcal{D}} x=1,000,000,000$ and $\max_{\mathcal{D}} x=1,000,000,001$, 10 digits are necessary to encode the boundaries of the data set, and only 5 digits remain available to encode the boundaries of the $10^9$ $\epsilon$-bins, which is not feasible.
A more critical limit is that the computer real values no longer behave as continuous  values within the range of this data set, as only $10^5$ distinct values can be encoded.
The "discrete" limit of computer real values is reached, and there is an important risk that the G-Enum method will treat this data set as a discrete one even if it comes from a continuous data distribution.

We suggest getting around this numerical limit by estimating the number of distinct values $n_d$ that can be encoded within the range of a data set and to exploit an accuracy parameter $E$ small enough to get on average at least $t_n=100$ distinct values per $\epsilon$-bin.
Let us first focus on the case where $0 \notin [\min_{\mathcal{D}} x; \max_{\mathcal{D}} x]$, for example $0 < \min_{\mathcal{D}} x < \max_{\mathcal{D}} x$.
The total number of distinct positive values of $\mathbb{R}^{(cr)}$ that can be encoded between $\doubleMin$ and $\doubleMax$ is $2^{64}/2 ~ \approx 9.10^{18}$  (cf. Section~\ref{sec:outliersWithFloatingPoint}).
Assuming that the density is almost constant in  $\log (\mathbb{R}^{(cr)})$ (cf. Section \ref{sec:logTransformation}), a raw approximation of the total number of distinct values 
in $[\min_{\mathcal{D}} x;\max_{\mathcal{D}} x]$ is
\begin{eqnarray*}
\mathrm{if} \; 0 < \min_{\mathcal{D}} x < \max_{\mathcal{D}} x,&& n_d([\min_{\mathcal{D}} x;\max_{\mathcal{D}} x]) \approx 2^{63} \frac{\log (\max_{\mathcal{D}} x) - \log (\min_{\mathcal{D}} x)}{\log \doubleMax - \log \doubleMin},\\
\mathrm{if} \; \min_{\mathcal{D}} x < \max_{\mathcal{D}} x< 0,&& n_d([\min_{\mathcal{D}} x;\max_{\mathcal{D}} x]) \approx n_d([-\max_{\mathcal{D}} x;-\min_{\mathcal{D}} x]),\\
\mathrm{if} \; \min_{\mathcal{D}} x < 0 < \max_{\mathcal{D}} x,&& n_d([\min_{\mathcal{D}} x;\max_{\mathcal{D}} x]) \approx n_d([\min_{\mathcal{D}} x;-\doubleMin])+ n_d([\doubleMin, \max_{\mathcal{D}} x]).
\end{eqnarray*}
If $n_d(\min_{\mathcal{D}} x , \max_{\mathcal{D}} x)/E < t_n$, the average number of distinct values that can be encoded per $\epsilon$-bin is below the threshold, and we replace $E=10^9$
by $E=\lceil 10^9 \times n_{d,\epsilon}([a;b])/t_b \rceil$ to get $\epsilon$-bins with enough distinct values per bin and keep a smooth continuous behavior of computer real values.

Note that when this numerical limit is reached, the separability of the values cannot be improved by splitting it into subsets, and we will consider the related data set as PWCH.

\subsection{Two-level heuristic}
\label{sec:outlierHeuristic}

If a data set $\mathcal{D}$ is practically well conditioned for histograms (PWCH), no significant loss of numerical precision is to be feared and we can compute a standard histogram.
Conversely, if it is PICH, we propose in Algorithm~\ref{algoOutliers} a two level heuristic that exploits the $\log_{\mathcal{D}}^{(cr)}$ function.

\begin{algorithm} [!htbp]
\caption{Two-level heuristic}
\label{algoOutliers}
\begin{algorithmic} [1]
  \REQUIRE $\mathcal{D}, E$
  \ENSURE $\mathbf{H}(\mathcal{D})$ 
	\STATE {\bf First level}
  \STATE compute the optimal histogram on $\log_{\mathcal{D}}^{(cr)}(\mathcal{D})$ to obtain a log-histogram $\mathbf{H}(\log_{\mathcal{D}}^{(cr)}(\mathcal{D}))$
	\STATE let $\mathbb{L}_\mathcal{D}=\{D_i\}$ be the list of adjacent data subsets $D_i$ of $\mathcal{D}$ related to the log intervals $i$ of $\mathbf{H}(\log_{\mathcal{D}}^{(cr)}(\mathcal{D}))$
	\STATE \COMMENT{Simplify the list $\mathbb{L}_\mathcal{D}$ by merging as much as possible adjacent intervals}
	\REPEAT
  \FOR{each pair of data subsets $(\mathcal{D}_i,\mathcal{D}_{i+1}) \in \mathbb{L}_\mathcal{D}$}
	\IF {$\mathcal{D}_i \cup \mathcal{D}_{i+1}$ is PWCH}
  \STATE replace $\mathcal{D}_i$ and $\mathcal{D}_{i+1}$ by $\mathcal{D}_i \cup \mathcal{D}_{i+1}$ in $\mathbb{L}_\mathcal{D}$
	\ENDIF
  \ENDFOR
	\UNTIL{no pair of adjacent subsets can be merged into a PWCH subset}
	\STATE \COMMENT{Split the remaining PICH data subsets of $\mathbb{L}_\mathcal{D}$}
  \FOR{each data subset $\mathcal{D}_i \in \mathbb{L}_\mathcal{D}$}
	\IF {$\mathcal{D}_i$ is PICH}
	\STATE split $\mathcal{D}_i$ into $k_i$ data subsets $\mathcal{D}_{i,k}$ using the splitting method described previously
  \STATE replace $\mathcal{D}_i$ by $\{\mathcal{D}_{i,k}\}_{1 \leq k \leq k_i}$ in $\mathbb{L}_\mathcal{D}$
	\ENDIF
  \ENDFOR
	
  \item[]
	\STATE {\bf Second level}
	\STATE \COMMENT{Compute sub histograms per data subset of $\mathbb{L}_\mathcal{D}$}
	\FOR{each data subsets $\mathcal{D}_i \in \mathbb{L}_\mathcal{D}$}
	\STATE build an optimal sub histogram $H_i$ of the data subset $\mathcal{D}_i$
	\ENDFOR
	\STATE \COMMENT{Concatenate the sub-histograms to initialize the output histogram $\mathbf{H}(\mathcal{D})}$
	\FOR{each sub histogram $H_i$}
	\STATE insert the intervals of $H_i$ in $\mathbf{H}(\mathcal{D})$
	\ENDFOR
	\STATE \COMMENT{Create boundary intervals to finalize the output histogram $\mathbf{H}(\mathcal{D})}$
	\FOR{each pair of sub histograms $(H_i, H_{i+1})$}
	\STATE let $Int_i^{last}$ be the last interval of $H_i$
	\STATE let $Int_{i+1}^{first}$ be the first interval of $H_{i+1}$
	\STATE let $Int_{i, i+1}^{empty}$ be the empty boundary between $Int_i^{last}$ and $Int_{i+1}^{first}$
	\STATE compute a boundary histogram $\mathbf{H}(\mathcal{D}_{i, i+1})$ for the data subset $\mathcal{D}_{i, i+1}=Int_i^{last} \cup Int_{i+1}^{first}$
	\STATE retrieve the boundary interval $Int_{i, i+1}^{boundary}$ of $\mathbf{H}(\mathcal{D}_{i, i+1})$ that includes $Int_{i, i+1}^{empty}$
	\STATE split the boundary data subset $\mathcal{D}_{i, i+1}$ into one, two or three intervals around $Int_{i, i+1}^{boundary}$
	\STATE replace the  two intervals $Int_i^{last}$ and $Int_{i+1}^{first}$ of $\mathbf{H}(\mathcal{D})$ by these new intervals
	\ENDFOR
\end{algorithmic}
\end{algorithm}

The first level splits the initial PICH data set into smaller data subsets, using a histogram built on the log transformation of the data. The resulting PWCH data subsets are then merged as far as possible, so as to obtain the largest possible PWCH data subsets for the second level. As for the PICH data subsets, they correspond to large intervals in the log-histogram. They are split into smaller data subsets according to the heuristic described at the end of this section, in order to obtain as few possible data subsets that are likely to be PWCH.

The second level produces sub histograms for all the PWCH data subsets resulting from the first level.
It is noteworthy that this two level heuristic could be used with any alternative histogram method in case of issues with outliers or with numerical precision limits.
In our case, we exploit the G-Enum method in the first level, with the sole purpose to split the whole numerical domain into sub domains. And we apply the G-Enum method in the second level, to build optimal sub histograms within each PWCH sub domain.
Boundary histograms are built between each pair of consecutive data subsets, by focusing on the last interval of the first data subset and the first interval of the second data subset. This allows to create a new interval to fill the boundary gap between the data subsets, and to finalize the global output histogram by replacing the two initial boundary intervals by one, two or three intervals. Note that each new built interval exploits the granularity parameter of its origin data subset.

The overall computational complexity of the two level algorithm is O$(n \log n)$, as all its components are based on algorithms with the same complexity, applied either to the whole data set or to the list of its data subsets.

We expect that the first level might help identifying outliers, as in the example of Section~\ref{sec:outliersLimits}.
Furthermore, we also expect that the suggested heuristic might be able to split numerical domains with heavy tail distribution into PWCH sub-domains.
As an example, let us consider an hypothetical data set with the weight of many organisms ranging from bacteria to insects and mammalians. The first level is likely to divide the numerical domain into at several sub-domains, that can then effectively be handled for the construction of specialized sub-histograms.

\paragraph{Splitting method for PICH data subsets obtained from the log space.}
Let $\mathcal{D}_i$ be a data subset related to an interval of the log-histogram $\mathbf{H}(\log_{\mathcal{D}}^{(cr)}(\mathcal{D}))$.
Let us first assume that $\mathcal{D}_i$ correspond to a positive interval $[a;b]$, with $0 < a < b$.
As $[a;b]$ is an interval of a histogram obtained  in the log space, we can assume a uniform density in $[\log a;\log b]$.
Our goal is to split $[\log a;\log b]$ into $k_i$ sub intervals of equal width in the log space, such that each sub interval $I_k$ is likely to be PWCH in the initial space.
Let $I_k=[\log a_{k-1}; \log a_k], 1 \leq k \leq k_i$, with $\log a_k=\log a + \frac{k}{k_i}(\log b - \log a)$.
Let $n_i$ be the number of data entries in $\mathcal{D}_i$.
As the density is assumed to be uniform in $[\log a;\log b]$, we can expect that the frequency of each sub interval $I_k$ is $n_{i, k} \approx n_i/k_i$.
In the initial space, the related data subsets are in intervals $[a_{k-1}; a_k]$, with $a_k/a_{k-1}$ being a constant as $(\log a_k-\log a_{k-1})$ is a constant.
As the uniform density in the log space translates into a decreasing density in the initial space, the most frequent bin in the initial space is likely to be the first one.
Using the threshold $t_E=\sqrt E \log E$ for the PICH criterion, the first bin of $[a_{k-1}; a_k]$ is $[a_{k-1}; a_{k-1} + (a_k-a_{k-1})/t_E]$.
Its frequency $n_{i, k}^\epsilon$ can be estimated using the uniform density assumption in the log space, according to 

\begin{eqnarray*}
n_{i, k}^\epsilon &\approx& n_{i, k} \frac{\log(a_{k-1}+(a_k-a_{k-1})/t_E)-\log a_{k-1}}{\log a_k-\log a_{k-1}},\\
   &\approx& \frac {n}{k_i} \; \frac{\log (1+(a_k/a_{k-1}-1)/t_E)}{\log a_k/a_{k-1}}.
\end{eqnarray*}
The frequency of the first bin of each data sub set is the same, so that the PICH criterion is likely to be triggered in the same way for all the data subsets related to the intervals $I_k, 1 \leq k \leq k_i$ .
We are searching for the smallest $k_i$, such that each data subset is PWCH, that is $n_{i, k}^\epsilon < \log n_{i, k}$.
As this might be complex to solve analytically, we suggest to solve this problem by dichotomy for $k_i \in \{2, n_i\}$ by computing all the values of $n_{i, k}^\epsilon$ and $\log n_{i, k}$. This can be done in O$(\log n_{i, k})$ computation time.

In the end, we can split our initial PICH data subset $\mathcal{D}_i$ into $k_i$ smaller data subsets that are likely to be PWCH.
The same method can be applied in the case of a data subset $\mathcal{D}_i$ with negative values, with $a < b < 0$. And in the case of a data subset with both positive and negative values, we can apply the method to both the negative and positive sub parts of the data subset.

\section{Preliminary analysis}
\label{sec:analysis}
In this section, we analyze some choices and properties relative to the two-level method.

\subsection{Threshold for being ill conditioned for histograms}
\label{sec:thresoldICH}

Let us first recall that the parameter $E$ of the G-Enum method is a fixed constant $E=10^9$ constrained by the limit $\intMax$ of computer integers.
For a data set $\mathcal{D}$ of size $n$, the ICH threshold is obtained for $gr(\mathcal{D}) = E$. 
In this section, we investigate on whether real world data sets are likely to be WCH given their size $n$ and the fixed constant $E=10^9$. Then, we evaluating the ICH and PICH criterions for the detection of ICH data sets.

\paragraph{Using a uniform distribution.}
Let us consider a data set $\mathcal{D}$ sampled from a uniform distribution on $[0,1]$.
The range of $\mathcal{D}$ is 1 while its expected precision is given by the expected minimum distance between its $n$ points, which is $$\frac{1}{n^2 -1}.$$

(See for example
\url{https://math.stackexchange.com/questions/1999612/average-minimum-distance-between-n-points-generate-i-i-d-with-uniform-dist}.)

We have 
\begin{eqnarray*}
gr(\mathcal{D}) = E &\Leftrightarrow& n^2-1 = E,\\
    &\Leftrightarrow& n \approx \sqrt{E}.
\end{eqnarray*}
The ICH threshold for a uniform distribution is attained for $E \approx \sqrt{n}$, that is for $n \approx 31,600$.

\paragraph{Using a Gaussian distribution.}
Let us now consider a Gaussian distribution as an example of peaked distribution.
As the Gaussian distribution may not be suitable for easily interpretable closed-form formulas, we focus instead on the binomial distribution, which can be approached asymptotically by a Gaussian distribution.

Let us consider a histogram with $(b+1)$ bins, where each bin $i$ is of width 1 and has a frequency equal to the binomial coefficient ${{b} \choose {i}}$. Let us denote $\mathcal{D}_b$ the artificial data set related to this histogram.
Let us assume that $n$ is a power of 2 with $n=2^b$.
We have
\begin{equation*}
n = (1+1)^b = \sum_{i=0}^b{{{b} \choose {i}}}.
\end{equation*}

The following formula enlightens the relation of the data set $\mathcal{D}_b$ with the Binomial distribution $B(n, p=1/2)$.
\begin{equation*}
n = 2^b \sum_{i=0}^b{{{b} \choose {i}}p^i(1-p)^{1-i}}.
\end{equation*}

As $b$ increases, the shape of this histogram converges to that of the normal distribution.
We assume a piecewise constant density per bin.
The range of $\mathcal{D}_b$ is $(b+1)$, which is the total width of the histogram.
As for the precision of $\mathcal{D}_b$, we assume that the central bin, which is the denser one, may be used to provide an approximation of the minimum difference between two consecutive values.
This central bin of index $b/2$ contains ${{b} \choose {b/2}}$ data entries.
Assuming a piecewise constant density within this bin, we apply the preceding results assuming a uniform distribution in the central bin.
We get 
\begin{eqnarray*}
rn(\mathcal{D}_b) &=& b+1,\\
pr(\mathcal{D}_b) &=& \frac{1}{{{b} \choose {b/2}}^2-1},\\
gr(\mathcal{D}_b) &=& (b+1)\left({{b} \choose {b/2}}^2-1\right).
\end{eqnarray*}

Using the Stirling formula $n!=\sqrt{2 \pi n} \left(\frac{n}{e}\right)^n + O(\frac{1}{n})$, we have
\begin{eqnarray*}
{{b} \choose {b/2}} &=& \frac {\sqrt{2 \pi b} \left(\frac{b}{e}\right)^b}
                              {\pi b \left(\frac{b}{2 e}\right)^b}  + O(\frac{1}{n}),\\
       &\approx& 2^b \sqrt{\frac{2}{\pi b}}.
\end{eqnarray*}

In terms of $n$ rather than $b$, we obtain
\begin{eqnarray*}
rn(\mathcal{D}_b) &=& \log_2{n}+1,\\
pr(\mathcal{D}_b) &\approx& \frac{\pi \log_2{n}}{2 n^2},\\
gr(\mathcal{D}_b) &\approx& \frac{2}{\pi} (1 + \frac{1}{\log_2{n}}) n^2.
\end{eqnarray*}

To get back to our initial problem of estimating the granular length for a Gaussian distribution, we have to apply some normalization.
The binomial distribution $X \sim \mathsf{B}(b, p)$ can be approximated by the normal distribution $X \sim \mathcal{N}(\mu=b p, \sigma = \sqrt{bp(1-p)})$.
Using $(X-\mu)/\sigma \sim \mathcal{N}(0, 1)$ and assuming a piecewise constant density per bin, let use consider the virtual data set $\mathcal{D}$ obtained by normalizing the data entries of $\mathcal{D}_b$ according to $Y=(X-\mu)/\sigma$ and generating data entries according to a uniform distribution within each bin.
These assumptions allow us to estimate the range, precision and granular length of $\mathcal{D}$, as approximations of these quantities for a Gaussian distribution.
Given that $b=\log_2 n$ and $p=1/2$, we get
\begin{eqnarray*}
rn(\mathcal{D}) &=& 2\frac{\log_2{n}+1}{\sqrt{\log_2 n}},\\
pr(\mathcal{D}) &\approx& \frac{\pi \sqrt{\log_2{n}}}{n^2},\\
gr(\mathcal{D}) &\approx& \frac{2}{\pi} (1 + \frac{1}{\log_2{n}}) n^2.
\end{eqnarray*}

The obtained approximation of the granular length for the Gaussian distribution is slightly smaller than that of the uniform distribution. Indeed, the approximation is likely to provide a lower bound since it relies on the assumptions that the precision of the data set can be evaluated from the central bin only and that the density is constant within this bin.

Let us finally approximate the ICH threshold:
\begin{eqnarray*}
gr(\mathcal{D}) = E &\Leftrightarrow& \frac{2}{\pi} (1 + \frac{1}{\log_2{n}}) n^2 \approx E,\\
	 	 &\Leftrightarrow& n \sqrt{1 + 1/log_2{n}} \approx \sqrt{\pi/2} \times \sqrt{E}.
\end{eqnarray*}

This approximation provides a ICH threshold for the Gaussian distributions, with the same order of magnitude as for the uniform distribution.
Let us remind that the ICH threshold is both translation and scale invariant and will be the same for any Gaussian distribution.
Using a numerical evaluation for $E=10^9$,  we get a threshold of $36,300$ data entries for Gaussian distributions.

\paragraph{Approach based on collisions.}
Another approach to evaluate the ICH threshold consists in evaluating when a histogram $\epsilon$-bin is likely to contain at least two distinct data entries.
For a data set $\mathcal{D}$ sampled from a uniform distribution, we are looking whether two distinct data entries among $n$ are likely to fall into the same $\epsilon$-bin among $E$.
This is known as the \emph{birthday problem}, which is to compute an approximate probability that in a group of $n$ people, at least two have the same birthday in a year with $E$ days.
This problem has been extensively studied in the literature (see Wikipedia for a summary of the results).
Below is an approximation of the threshold $n$ for having a probability above $\frac{1}{2}$ that one $\epsilon$-bin contains two data entries:
\begin{eqnarray*}
n &\approx& 1/2 + \sqrt{1/4+ 2 \log{2} \times {E}}.
\end{eqnarray*}
For $E=10^9$, the threshold for the birthday problem is $37,200$.

Altogether, the three approaches based on the uniform distribution, the Gaussian distribution and the detection of bin collision provide the same order of magnitude, $n \approx \sqrt{E}$, that is about $30,000$ data entries per data set for $E=10^9$.

\paragraph{Experimental evaluation of the ICH property.}
To confirm the theoretical insights provided in the previous sections, we perform numerical experiment to evaluate the empirical granular length of data sets sampled from a uniform or a Gaussian distribution.
We generate $10,000$ data sets with $n \in [1,000; 100,000]$ with increasing sizes using a geometric increment of $\sqrt{2}$.
As expected, the empirical granular length suffers from a very large standard deviation. Indeed, the empirical standard deviation is between 10 and 50 times larger than the mean, and the mean itself is between 5 and 20 times larger than the median.
We thus chose to report the median of the granular lengths in Figure~\ref{fig:study_ICH_gr}.
We also report the theoretical approximation of the granular length for the uniform distribution, as well its lower bound for the Gaussian distribution.

\begin{figure}[htbp!]
\begin{center}
\includegraphics[trim={0 2.5cm 0 1.5cm 0}, width=0.70\columnwidth]{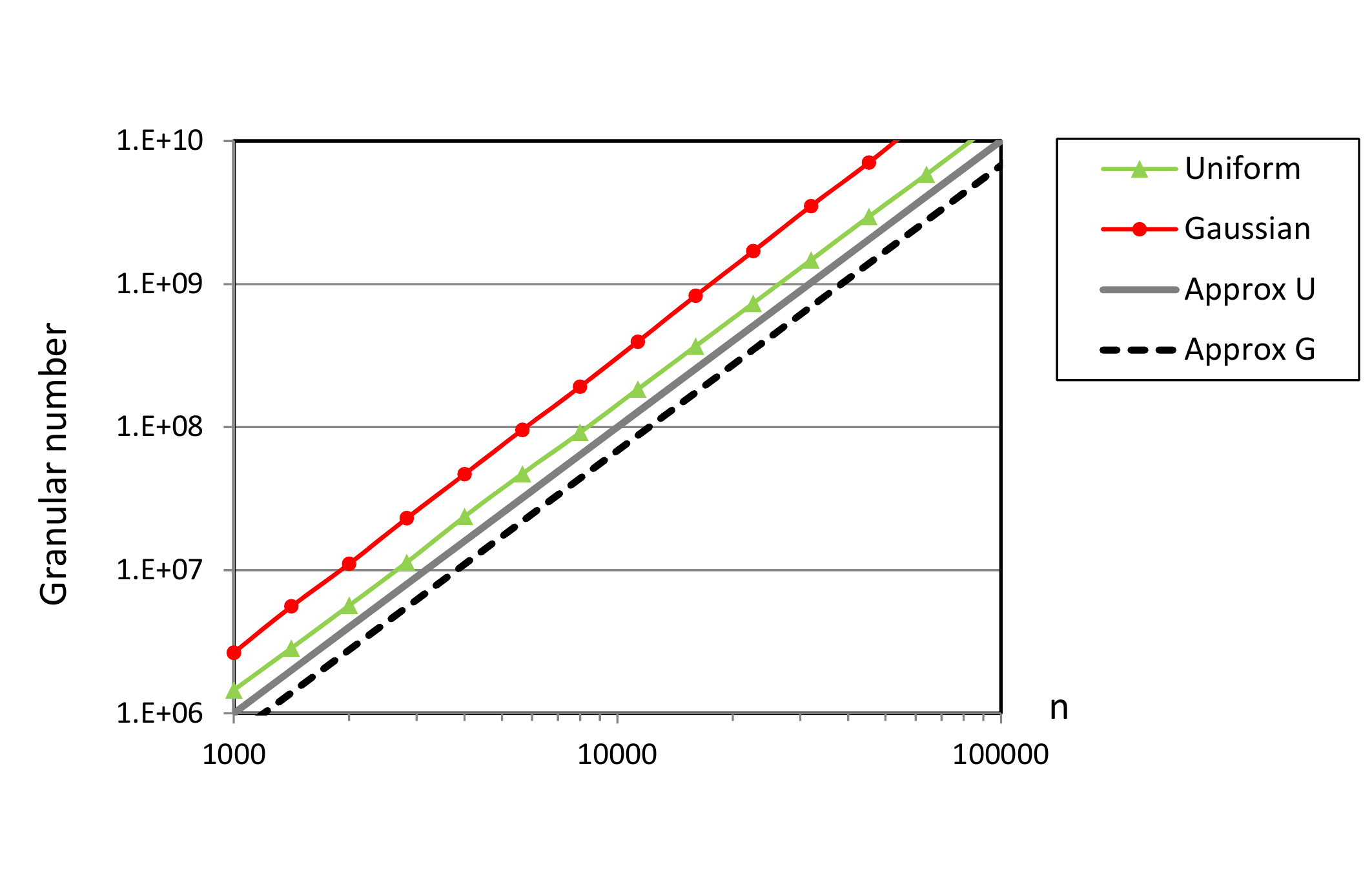}
\end{center}
\caption{Median granular length for the uniform and Gaussian distributions.}
\label{fig:study_ICH_gr}
\end{figure}

For the uniform distribution, the empirical median of the granular length is close to its approximation, with a ratio of around $1.5$.
For the Gaussian distribution, the empirical median of the granular length is about 5 times the approximation, which is a lower bound as expected.
And the Gaussian distribution has a granular length between 2 and 3 times that of the flat uniform distribution.
Overall, this confirms that the order of magnitude of the granular length grows as the square of the size of the data set. The median of the granular length goes beyond the $\epsilon$-bin length of histograms ($E=10^9$) beyond $n \approx 20,000$, but this criterion suffers from a tremendously large variance.

We also collect the number of collisions, that is the number of data entries that share their bin with another data entry of different value and cannot be separated using a histogram.
We report in Figure~\ref{fig:study_ICH_collisions} the mean and standard deviation of the collision numbers
This criterion is more stable than the granular number, but still even small data set may have some collisions.
Not surprisingly, the collision number is larger with the Gaussian than with the uniform distribution, as both the Gaussian range is larger and its precision is likely to be smaller because of the higher density in the Gaussian peak.

\begin{figure}[htbp!]
\begin{center}
\includegraphics[trim={0 2.5cm 0 1.5cm 0}, width=0.70\columnwidth]{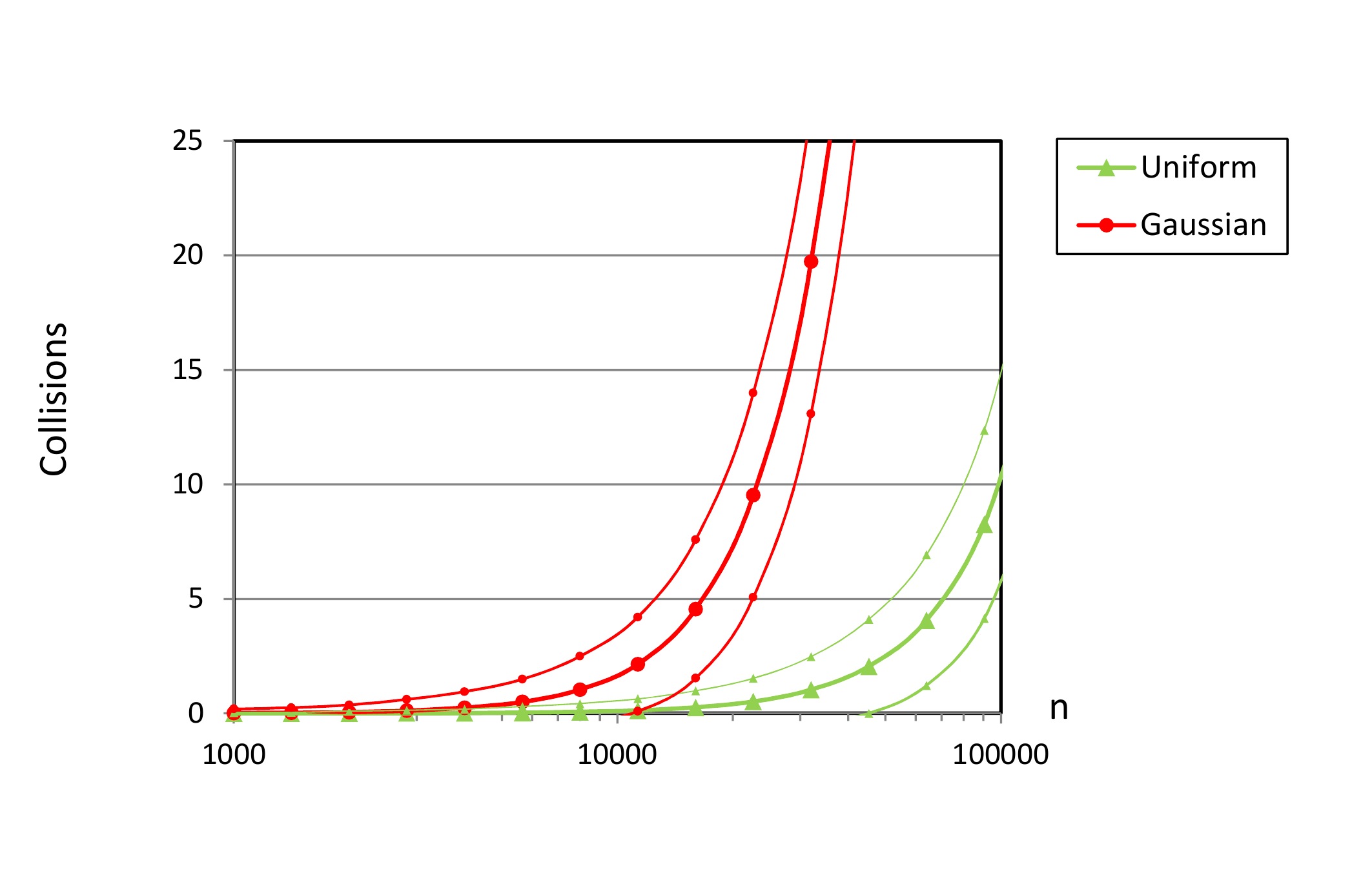}
\end{center}
\caption{Mean collision number for the uniform and Gaussian distributions.}
\label{fig:study_ICH_collisions}
\end{figure}

Finally, we study the behavior of two criterions for detecting the ICH property of a data set:
\begin{enumerate}
	\item ICH: number of collisions is greater or equal than 1,
	\item RICH: (robust ICH) number of collisions is greater than $\log n$,
\end{enumerate}

We report in Figure~\ref{fig:study_ICH_criterions} the proportion of data sets (among $10,000$) that are detected as ICH, according to the ICH and RICH criterions.
The ICH criterion exhibits a very large variance and results in a rather small threshold for the size of the data  detected as ICH (between $10,000$ and $30,000$ for a probability $50\%$ of detection).
The RICH criterion is more robust at the expense of a potential small loss of precision. 
The variance is far smaller with almost no ICH detection for sizes below $10,000$, larger detection size ($\approx 15,000$) in the case of the Gaussian distribution and far larger detection size ($\approx 100,000$) in the case of the uniform distribution.

\begin{figure}[htbp!]
\begin{center}
\includegraphics[trim={0 2.5cm 0 1.5cm 0},width=0.49\columnwidth]{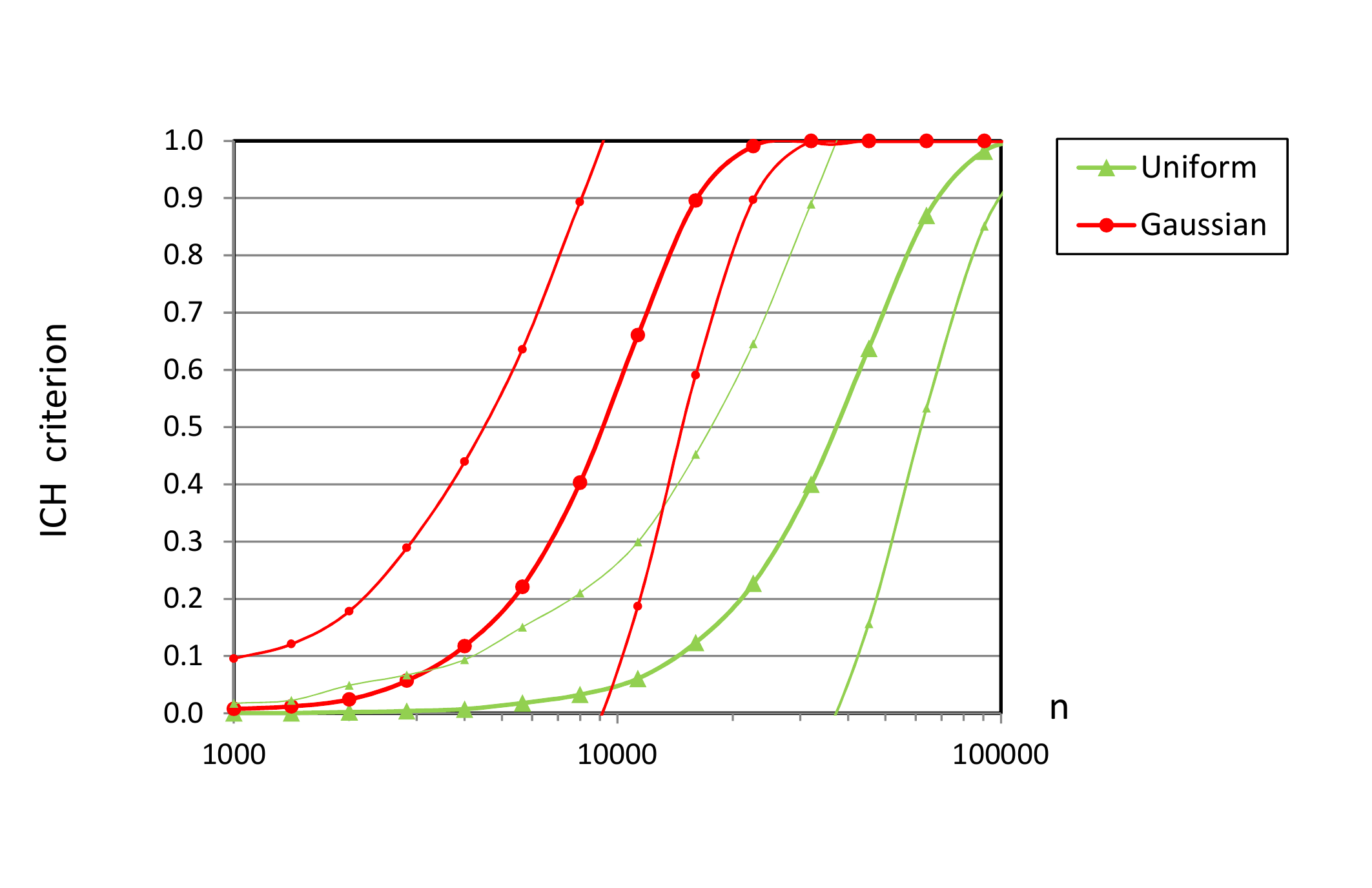}
\includegraphics[trim={0 2.5cm 0 1.5cm 0},width=0.49\columnwidth]{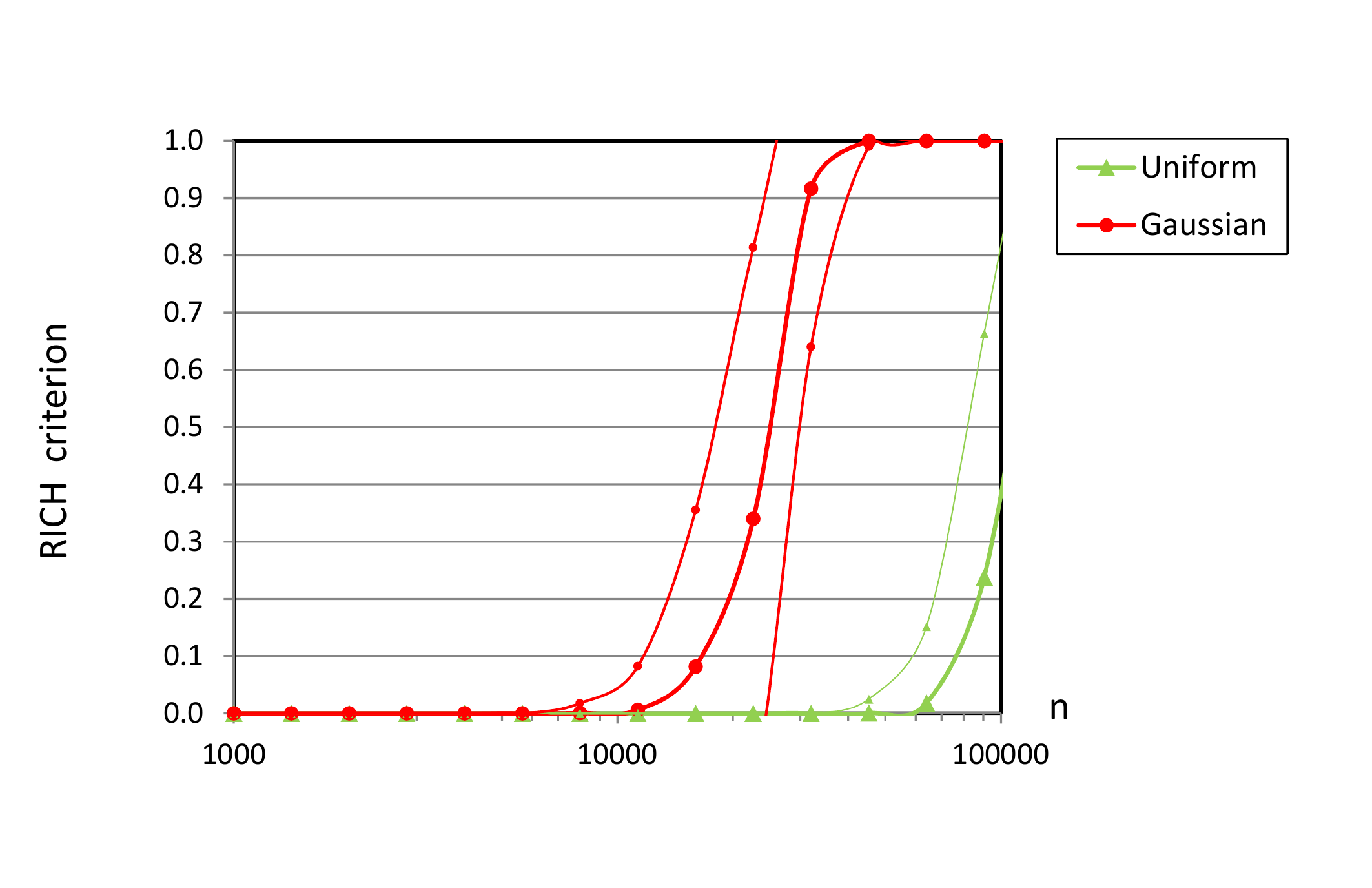}
\end{center}
\caption{Probability of detection of ICH data sets using the ICH and RICH criterions.}
\label{fig:study_ICH_criterions}
\end{figure}

\paragraph{Experimental evaluation of the PICH criterion.}

The PICH criterion introduced in Section~\ref{sec:wchProperty} is triggered if at least one colliding bin within a granularized histogram with $t_E=\sqrt{E}\log{E}$ bins contains more than $\log{n}$ data entries. 
We report in Table~\ref{table:PICH} the minimum data set size $n$ for a probability of $50\%$ of detection of the ICH property for different threshold $t_E$.

\begin{table}[!htb]
\caption{PICH criterion: minimum data set size $n$ for a probability of $50\%$ of detection}
\begin{center} \begin{tabular}{ccc}
\hline\\
$t_E$ & Uniform distribution & Gaussian distribution \\
\hline\\
$E$ & $> 10^9$ & $\approx 6.\; 10^7$ \\
$\sqrt{E} \log{E}$ &  $\approx 6.\; 10^4$ &  $\approx 2.\; 10^6$ \\
$\sqrt{E}$ & $\approx 4.\;10^3$ &  $\approx 8.\; 10^4$\\
\hline
\end{tabular}\end{center}
\label{table:PICH}
\end{table}

For $t_E=E$, the PICH criterion is triggered only for very large data sets.
For $t_E=\sqrt{E}$, the PICH criterion is triggered for rather small data sets w.r.t. usual real world data sets.
Using $t_E=\sqrt{E}\log{E}$ looks a good trade-off. The PICH criterion is not likely to be triggered too often and as $\sqrt{E}\log{E} \ll E$, the G-Enum algorithm is likely to find an optimal granularity $G$ far below $E$ and to get a stable behavior.

\begin{figure}[htbp!]
\begin{center}
\includegraphics[trim={0 2.5cm 0 1.5cm 0},width=0.6\columnwidth]{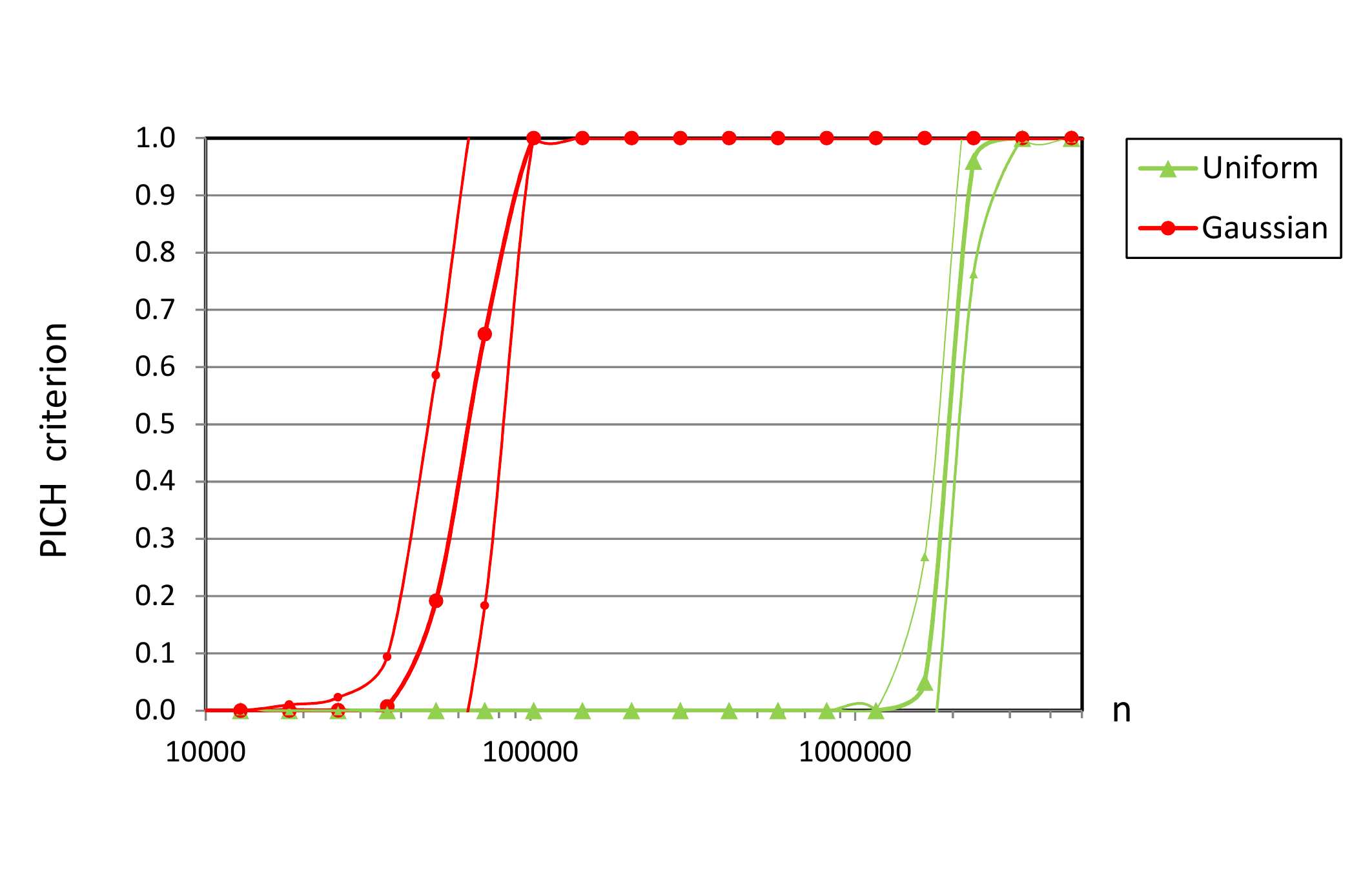}
\end{center}
\caption{Probability of detection of ICH data sets using the PICH criterion.}
\label{fig:study_PCH_criterions}
\end{figure}

We report in Figure~\ref{fig:study_PCH_criterions} the results of $10,000$ experiments with the mean and standard deviation of the PICH criterion for data sets from size $n=10,000$ to $50,000,000$.
This confirms that the PICH criterion is triggered for data sets of large enough size and has a moderate variance.

\paragraph{Synthesis.}
The ICH property of data set is triggered for data set sizes $n$ of around the square root of the number $E$ of $\epsilon$-bins of a histogram.
For $E=10^9$, this gives a size threshold of a few tens of thousands, but with a tremendously large variance.
The PICH criterion requires that at least one bin contains at least $t_c=\log{n}$ colliding data entries for $t_E=\sqrt{E} \log{E}$ elementary bins.
This more robust criterion pushes the threshold up to sizes of tens of thousands, at the expense of a negligible loss of precision of $\log{n}/n$ for the bounds of the intervals.
This makes the PICH criterion well suited for the use in the two level heuristic presented in Section~\ref{sec:outlierHeuristic}.

\subsection{Histogram on initial and log-transformed data sets}

In this section, we illustrate the impact of the log-transformation of data sets on the construction and visualization of histograms.

\begin{figure}[htbp!]
\begin{center}
\includegraphics[width=0.4\columnwidth]{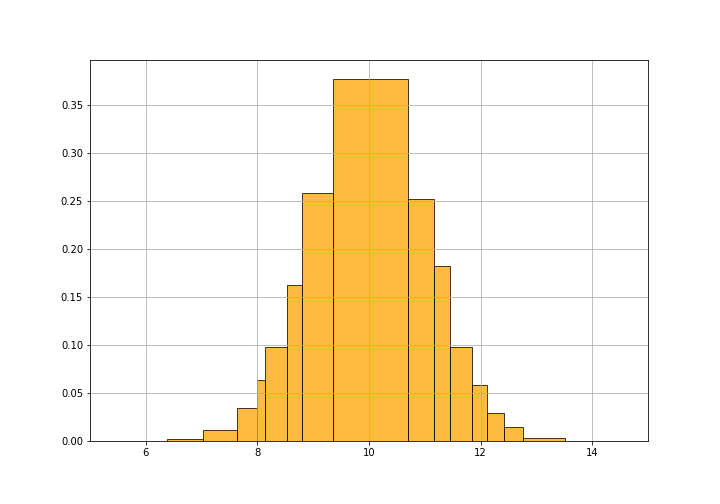}
\includegraphics[width=0.4\columnwidth]{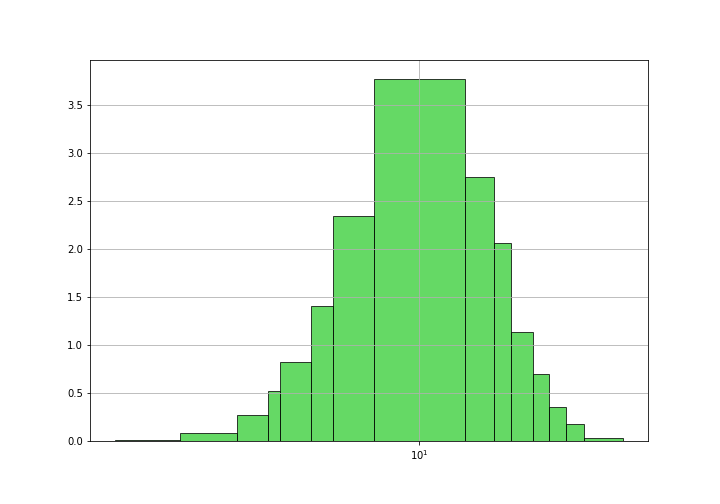}
\end{center}
\caption{Histograms built from $\mathcal{D}$ and visualized on the standard domain (left) and $\log^{(cr)}$ domain (right), for the Gaussian distribution $G(\mu=10, \sigma=1)$.}
\label{fig:histogramsG_10_1}
\end{figure}

Let $\mathcal{D}$ of size $n=10,000$ generated according to a Gaussian distribution $G(\mu=10, \sigma=1)$.
Figure~\ref{fig:histogramsG_10_1} presents a histogram built from $\mathcal{D}$  and visualized on the initial domain (left) and on the log transformed domain (right). For ease of read, the ticks and their label on the $X$ axis are reported with their initial values in $\mathcal{D}$.
On the left, the histogram is nicely balanced, as expected for a Gaussian distribution.
On the right, the histogram is unbalanced, which naturally comes from the log transformation of the data.

\begin{figure}[htbp!]
\begin{center}
\includegraphics[width=0.4\columnwidth]{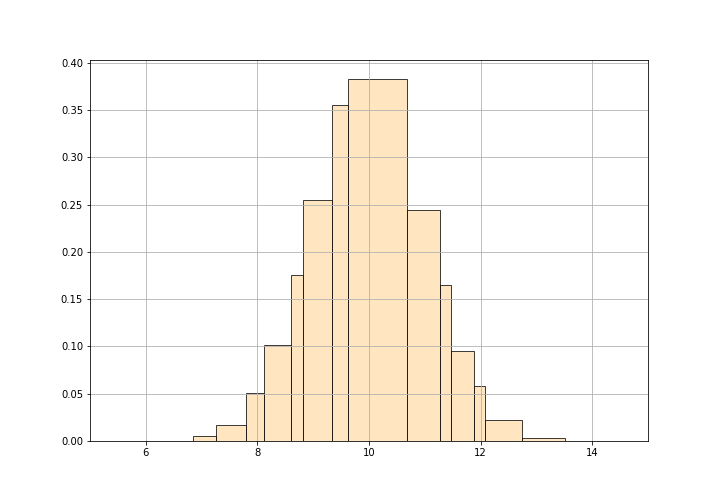}
\includegraphics[width=0.4\columnwidth]{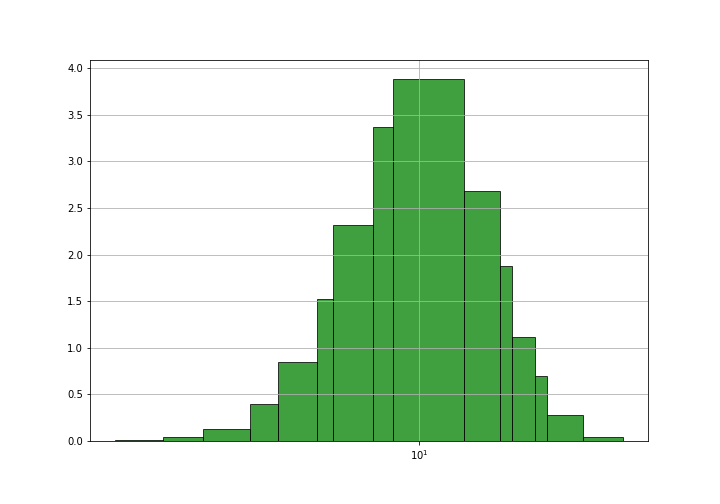}
\end{center}
\caption{Histograms built from $\log^{(cr)}(\mathcal{D})$ and visualized on the standard domain (left) and $\log^{(cr)}$ domain (right), for the Gaussian distribution $G(\mu=10, \sigma=1)$.}
\label{fig:histogramsLogG_10_1}
\end{figure}

Conversely, Figure~\ref{fig:histogramsLogG_10_1} presents a histogram built from $\log^{(cr)}(\mathcal{D})$ and visualized on the initial domain (left) and on the log transformed domain (right).
On the right, the histogram is unbalanced as expected in the logarithmic domain.
On the left, the histogram on the initial domain is awkwardly balanced because it was built on the other domain.

Let us note the flat upper lines of the histogram bars are consistent with the underlying piecewise constant density estimation. To maintain this consistency, histograms build on the initial domain should be represented on the logarithmic domain using upper lines with a logarithmic slope.
Conversely, histograms build on the logarithmic domain should be represented on the initial domain using upper lines with a decreasing slope.

\begin{figure}[htbp!]
\begin{center}
\includegraphics[width=0.4\columnwidth]{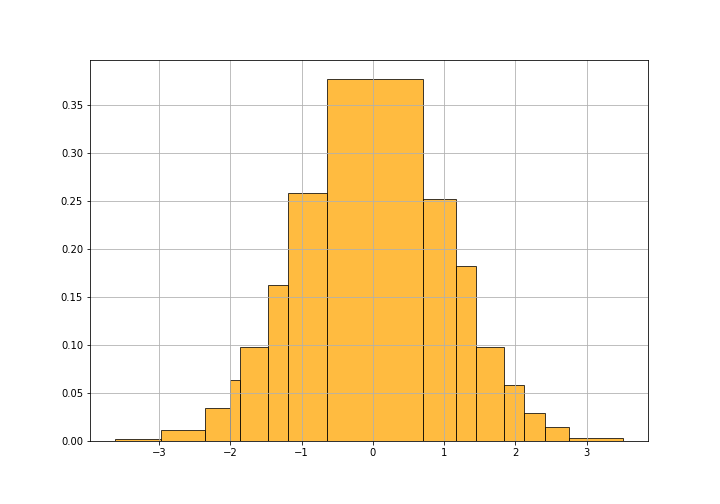}
\includegraphics[width=0.4\columnwidth]{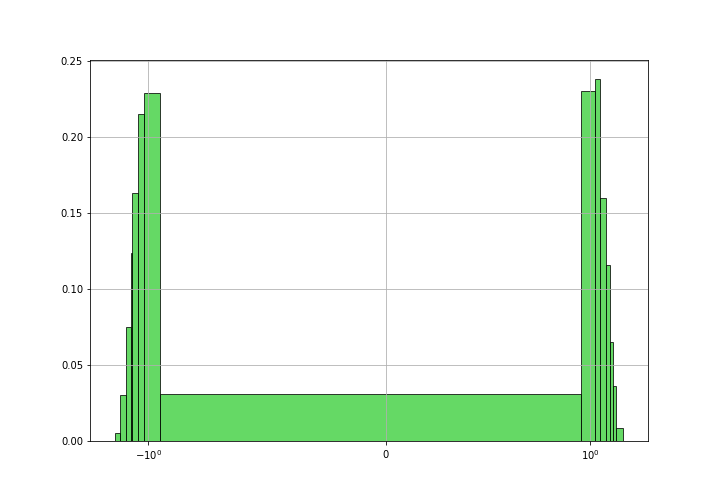}
\end{center}
\caption{Histograms built from $\mathcal{D}$ and visualized on the standard domain (left) and $\log^{(cr)}$ domain (right), for the Gaussian distribution $G(\mu=0, \sigma=1)$.}
\label{fig:histogramsG_0_1}
\end{figure}

\begin{figure}[htbp!]
\begin{center}
\includegraphics[width=0.4\columnwidth]{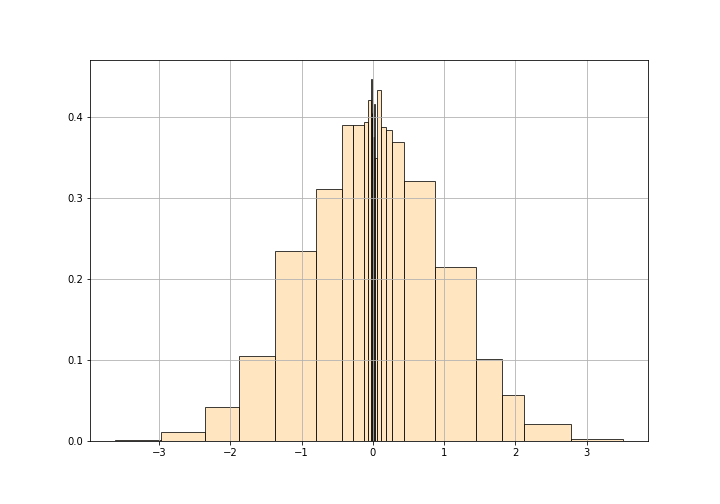}
\includegraphics[width=0.4\columnwidth]{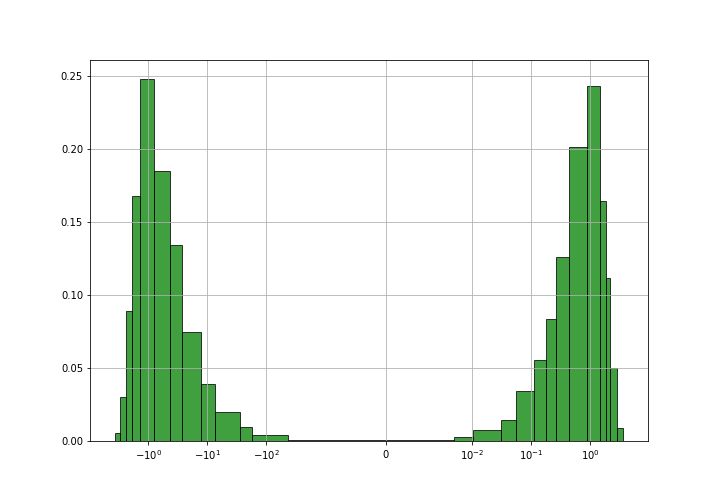}
\end{center}
\caption{Histograms built from $\log^{(cr)}(\mathcal{D})$ and visualized on the standard domain (left) and $\log^{(cr)}$ domain (right), for the Gaussian distribution $G(\mu=0, \sigma=1)$.}
\label{fig:histogramsLogG_0_1}
\end{figure}

Figure~\ref{fig:histogramsG_0_1} and Figure~\ref{fig:histogramsLogG_0_1} present the same visualizations in the case of a Gaussian distribution $G(\mu=0, \sigma=1)$ centered on 0, which is a singular point for the log transformation. The differences between the domain where the histogram is build and the one where it is visualized are now highly contrasted.

\medskip
To summarize, the log transformation of the data has the advantage of being usable for visualization of any data set, with either negative, null or positive data.
However, although any domain might be convenient for visualization purposes, histograms should be built on their own data domain.
In the case of the two-level heuristic, the log transformed domain is used only because of its appealing property to divided the initial data set into data subsets. The output histograms are built on the initial domain.

\subsection{Scale and translation invariance in $\mathbb{R}$ and $\mathbb{R}^{(cr)}$}

\begin{theorem}
\label{th:realLinearInvariance}
The optimal histogram built from the linear transformation of a data set $\mathcal{D} \subset \mathbb{R}$ is the same as the linear transformation of the optimal histogram built from $\mathcal{D}$, with the linear transformation of its interval bounds.
\end{theorem}

Theorem~\ref{th:realLinearInvariance} states that optimal histograms built from data sets in $\mathbb{R}$ are invariant under linear transformation $f_{a,b}(x) = a x + b$ of the data. 
This nice property stems from the existence of a bijection between the space of histograms that can be built from a data set $\mathcal{D}$ and the space of histogram that can be built from $f_{a,b}(\mathcal{D})$.

\begin{theorem}
\label{th:computerNoLinearInvariance}
The optimal histogram built from the linear transformation of a data set $\mathcal{D} \subset \mathbb{R}^{(cr)}$ is not always the same as the linear transformation the optimal histogram built from $\mathcal{D}$.
\end{theorem}

However, when it comes to computer real values with floating-point representation, Theorem~\ref{th:computerNoLinearInvariance} states that this is not longer true.
Indeed, there is no longer a bijection between $\mathcal{D}$ and $f_{a,b}(\mathcal{D})$ nor between their related space of histograms. Some of the potential impacts are given below as examples.
\begin{itemize}
  \item for $b=0$,
	\begin{itemize}
		\item if $a$ is too small, all values in $f_{a,b}(\mathcal{D})$ are underflow,
		\item if $a$ is too large, all values in $f_{a,b}(\mathcal{D})$ are overflow,
	\end{itemize}
	\item for $a=0$,
	\begin{itemize}
		\item if $b$ is too small, all values $x \in \mathcal{D}$ are such that $x+b=x$, resulting in $f_{a,b}(\mathcal{D})=\mathcal{D}$,
		\item if $b$ is too large, all values $x \in \mathcal{D}$ are such that $x+b=b$, resulting in $f_{a,b}(\mathcal{D})=\{b\}$.
	\end{itemize}
\end{itemize}

\medskip
Floating-point values allow an acceptable behavior on a wide range of real world applications, but their limits can produce unexpected results, as in the case of data sets with outliers.
Even methods with well grounded theoretical foundations may fail in some simple cases. Accounting for the limits of floating-point representation may help pushing the limits of these methods.

\section{Experimental evaluation}
\label{sec:evaluation}
In this section, we evaluate the impact of the two-level method on the quality of the built histograms using artificial data sets.
The quality of an histogram can be evaluated using a statistical distance between the underlying probability distribution and the histogram considered as a piecewise constant density estimator. Among the usual statistical distances are the Kullback-Leibler divergence, the Hellinger distance or the mean square error. However, some of these measures assume that the probability distribution has a density, which is disputable in the case of outliers. The scale of these measures may vary a lot depending on the data and the results are difficult to compare and interpret.
In the experiments, we rather exploit the number of intervals as an indirect measure of the quality of the histograms. Indeed, as the G-Enum method is regularized, it is not likely to overfit the data and the number of intervals appears to be highly correlated with the accuracy of the retrieved patterns. Lastly, this very simple measure is suitable for easy comparisons and interpretation.

\subsection{Resistance to one outlier}

The objective of this experiment is to evaluate the impact of one outlier on the quality of the built histograms.
We exploit a data set of size $n=10,000$ generated from a Gaussian distribution $G(\mu=1, \sigma=0.1)$.
We add one outlier with value $v_{out} = 2^i$ and consider all the 35 values from $v_{out} = 1$ to $v_{out} = 2^{34} \approx 1.7\;10^{10}$.
The experience is repeated 100 times, which represents 3,500 data sets.

\begin{figure}[htbp!]
\begin{center}
\includegraphics[width=0.7\columnwidth]{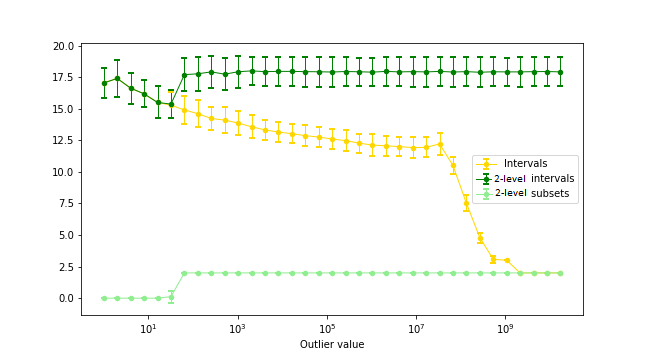}
\end{center}
\caption{Number of intervals obtained using or not the two-level method, for the Gaussian distribution $G(\mu=1, \sigma=0.1)$ and one outlier}
\label{fig:oneOutlier}
\end{figure}

Figure~\ref{fig:oneOutlier} reports the mean and standard deviation of the number intervals obtained using or not the two-level method. The number of data subsets considered by the method is reported as well.
For $v_{out}=1$, there are no outliers and the retrieved histogram contains around 17 interval to approximate the Gaussian distribution.
For small values of $v_{out}$, it is not clear whether $v_{out}$ is a point in the tail of the Gaussian distribution or an outlier value. Both the standard and the two-level methods build the same histograms with slightly less intervals, down to around 15 intervals for $v_{out}=32$.
For larger values of $v_{out}$, the standard method build less and less intervals, down to 12 intervals for 
$v_{out}\approx 3.\;10^7$, before a fast drop down to 2 intervals when all the Gaussian data entries collide in the first histogram bin. Conversely, the two-level method splits the data into two data subsets for $v_{out}> 32$ and builds a histogram consisting of about 18 intervals, 17 for the Gaussian data and one for the outlier.

\begin{figure}[htbp!]
\begin{center}
\includegraphics[width=0.32\columnwidth]{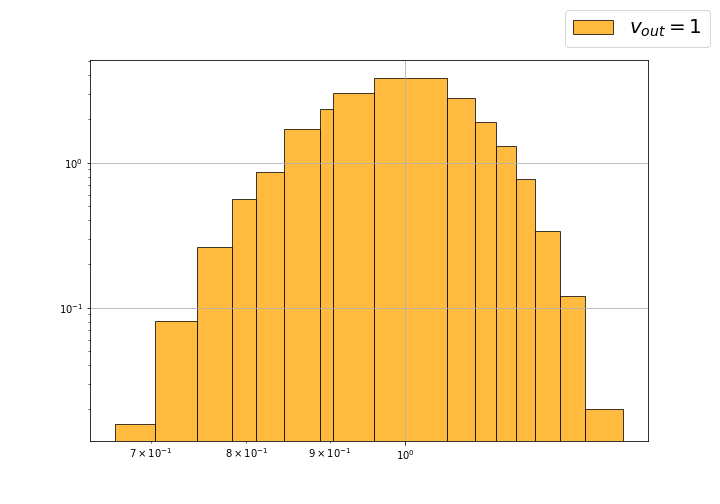}
\includegraphics[width=0.32\columnwidth]{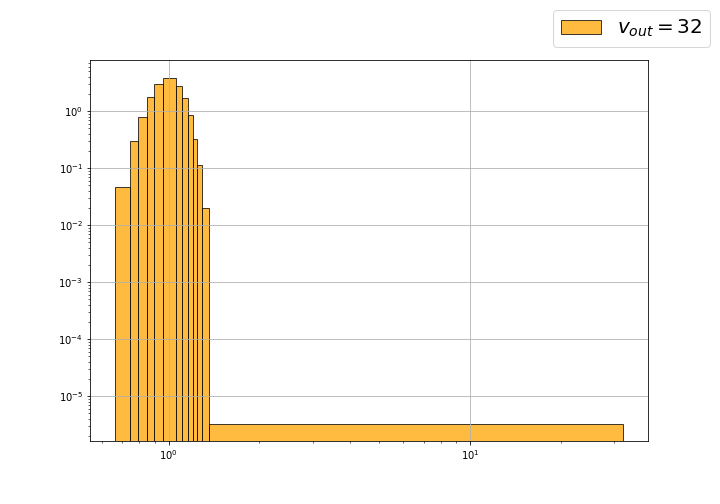}
\includegraphics[width=0.32\columnwidth]{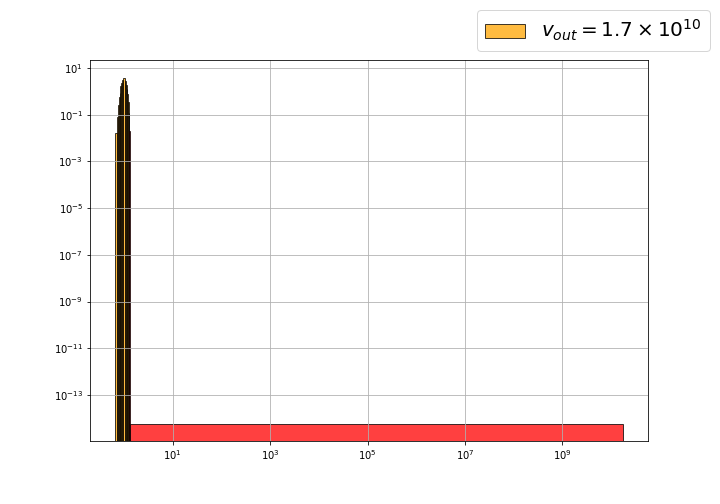}
\end{center}
\caption{Histograms obtained using the two-level method for the Gaussian distribution $G(\mu=1, \sigma=0.1)$ and different values of outlier, on the $\log \times \log$ scale. The boundary intervals are displayed in red in the case of several data subsets}
\label{fig:oneOutlierHistogramsLogLog}
\end{figure}

\begin{figure}[htbp!]
\begin{center}
\includegraphics[width=0.32\columnwidth]{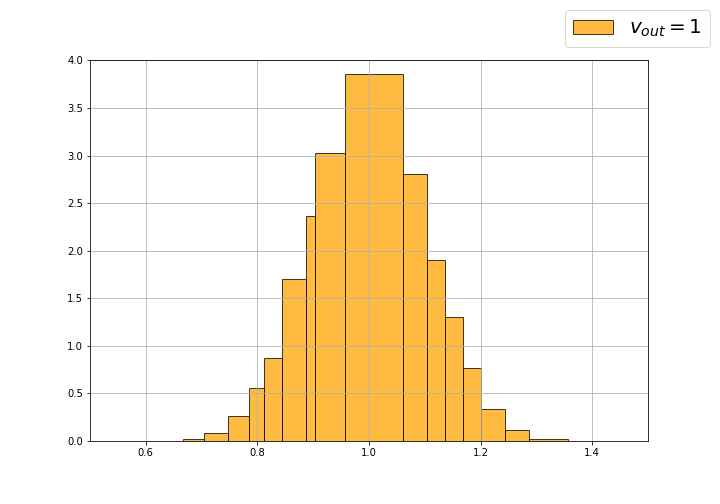}
\includegraphics[width=0.32\columnwidth]{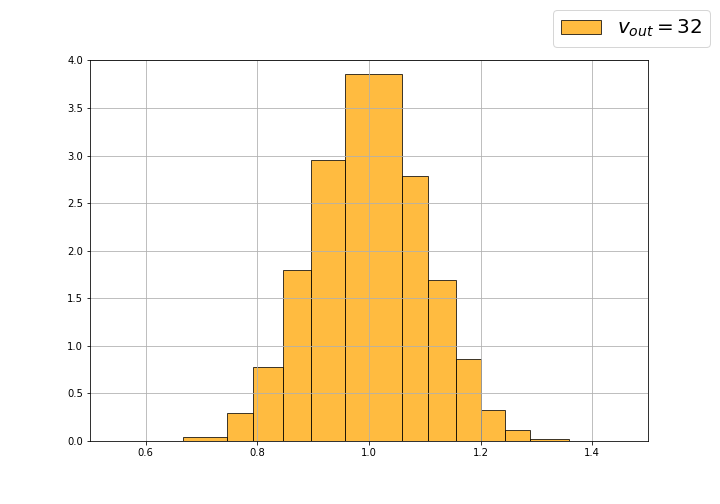}
\includegraphics[width=0.32\columnwidth]{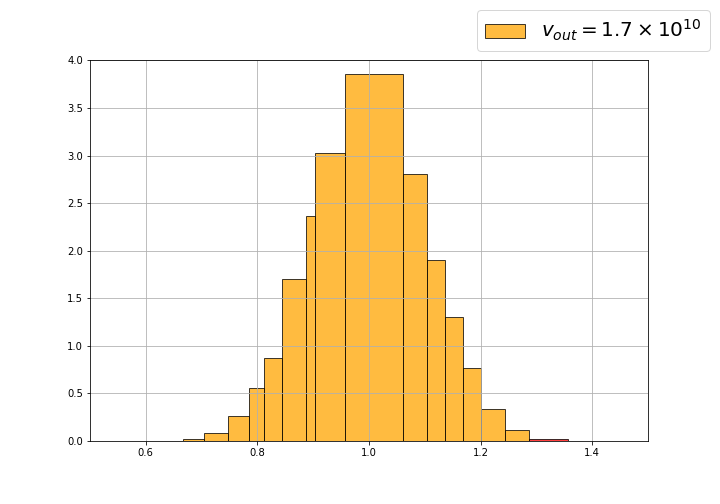}
\end{center}
\caption{Histograms obtained using the two-level method for the Gaussian distribution $G(\mu=1, \sigma=0.1)$ and different values of outlier, with a focus on $X \in [0.5; 1.5]$}
\label{fig:oneOutlierHistograms}
\end{figure}

The histograms built using the two-level method are displayed for $v_{out}=1, 32$ and $10^{10}$ using a $\log \times \log$ scale in Figure~\ref{fig:oneOutlierHistogramsLogLog} and using the standard scale with a focus on the Gaussian data in Figure~\ref{fig:oneOutlierHistograms}. The boundary intervals are displayed in red in the case of several data subsets. This shows that the main Gaussian distribution is correctly approximated whatever be the outlier value.

\subsection{Resistance to a distribution of outliers}

The objective of this experiment is to evaluate the impact of a distribution of outliers on the quality of the build histograms.
We exploit a data set of size $n=10,000$ generated from a Gaussian distribution $G(\mu=1, \sigma=0.1)$.
We add 100 outliers generated from a Gaussian distribution $G(\mu_O=1, \sigma_O=)$ with value $\sigma_O =  2^i \times 10^{-10}, 0 \leq i \leq 67$ and consider all the 68 values from $\sigma_O = 10^{-10}$ to $\sigma_O = 2^{67} \times  10^{-10} \approx 1.5\;10^{10}$.
The experience is repeated 100 times, which represents 6,800 data sets.

\begin{figure}[htbp!]
\begin{center}
\includegraphics[width=0.7\columnwidth]{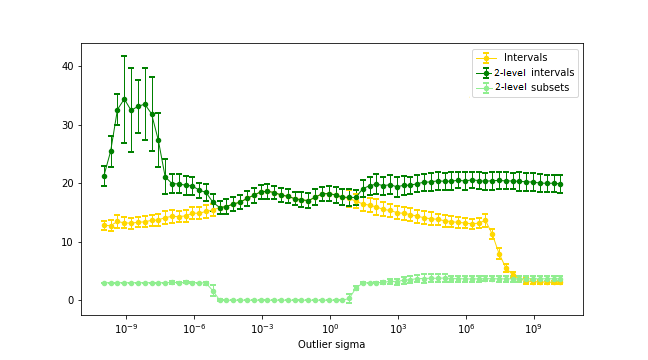}
\end{center}
\caption{Number of intervals obtained using or not the two-level method, for the Gaussian distribution $G(\mu=1, \sigma=0.1)$ and 100 outlier distributed according to a Gaussian distribution with same mean and a wide range of standard deviations}
\label{fig:distributionOutliers}
\end{figure}

Figure~\ref{fig:distributionOutliers} reports the mean and standard deviation of the number intervals obtained using or not the two-level method, as well as the number of involved data subsets.
Interestingly, three regimes can be observed with small transitions between them.
For $\sigma_O \in [10^{-5}; 3.5]$, the distribution of the outliers cannot be distinguished from the main Gaussian distribution and both the standard and two-level methods build the same histogram with 16 to 18 intervals.
For $\sigma_O \leq 10^{-6}$, both methods identify the distribution of outliers, which essentially reduces to one peak interval in the center of the main Gaussian data (cf. Figure~\ref{fig:distributionOutliersHistogramsLogLog}).
Contrary to the standard method, the two-level method splits the data set into three subsets, one for the central distribution of outliers surrounded by two other ones for the main Gaussian distribution. Three independent histograms are built for each subset, resulting in altogether, around 20 to 35 intervals.
For $\sigma_O \geq 6$, the standard method fails to correctly summarize the distribution when $\sigma_O \rightarrow \infty$. The two-level method splits the data set into three to four subsets, one for the main Gaussian data distribution distribution of outliers and the other ones for the outliers. Altogether, around 20 intervals are built.

\begin{figure}[htbp!]
\begin{center}
\includegraphics[width=0.4\columnwidth]{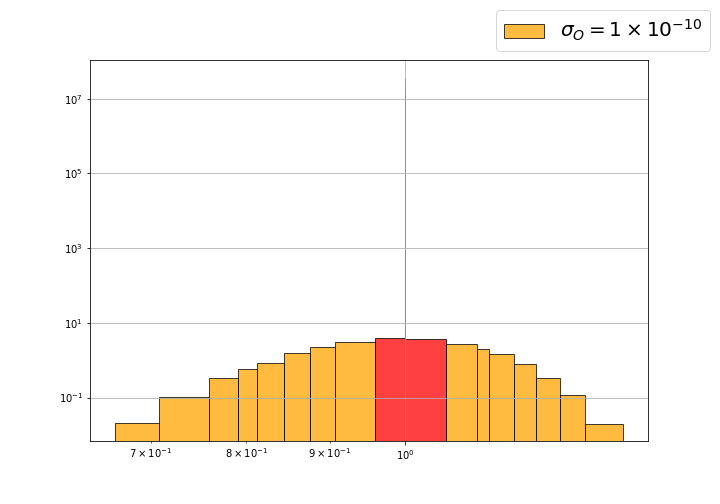}
\includegraphics[width=0.4\columnwidth]{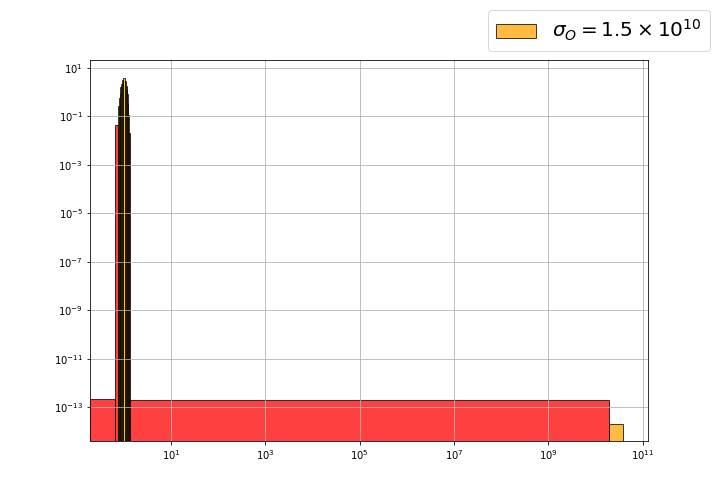}
\end{center}
\caption{Histograms obtained using the two-level method for the Gaussian distribution $G(\mu=1, \sigma=0.1)$ and different distributions of outliers, on the $\log \times \log$ scale. The boundary intervals are displayed in red in the case of several data subsets}
\label{fig:distributionOutliersHistogramsLogLog}
\end{figure}

\begin{figure}[htbp!]
\begin{center}
\includegraphics[width=0.4\columnwidth]{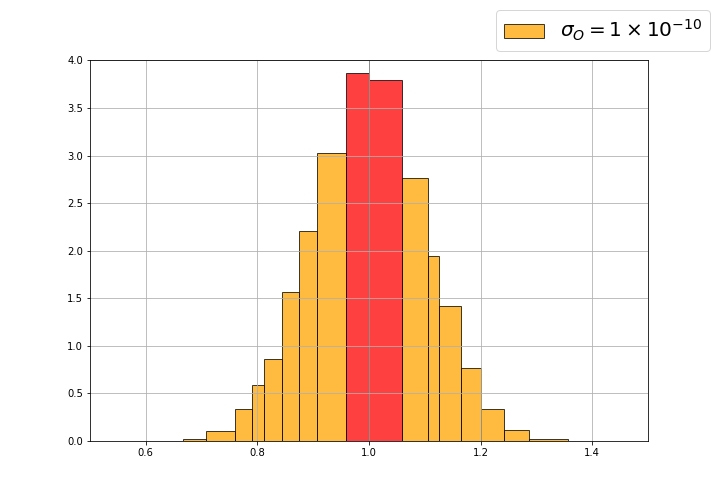}
\includegraphics[width=0.4\columnwidth]{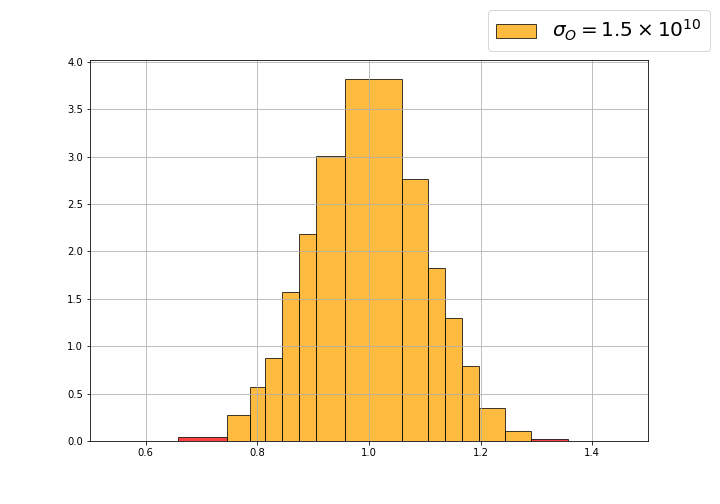}
\end{center}
\caption{Histograms obtained using the two-level method for the Gaussian distribution $G(\mu=1, \sigma=0.1)$ and different distributions of outliers, with a focus on $X \in [0.5; 1.5]$}
\label{fig:distributionOutliersHistograms}
\end{figure}

The histograms built using the two-level method are displayed for $\sigma_o=1 \times 10^{-10}$ and $1.5 \times 10^{10}$ using a $\log \times \log$ scale in Figure~\ref{fig:distributionOutliersHistogramsLogLog} and using the standard scale with a focus on the main Gaussian data in Figure~\ref{fig:distributionOutliersHistograms}.
This shows that the main Gaussian distribution is correctly approximated for a very large range of standard deviations of the outlier distribution.

\subsection{Data set with a heavy tail distribution}

The objective of this experiment is to evaluate the behavior of the method in the case of a data set with a heavy tail distribution.
We exploit a data set of size $n=20,000$ generated from a equidistributed mixture of two Gaussian components $G(\mu_1=1, \sigma_1=\mu_1/10)$ and $G(\mu_2, \sigma_2=\mu_2/10)$, where $\mu_2 = 2^i$.
We consider all the 35 values of from $\mu_2=1$ to $\mu_2 = 2^{34} \approx 1.7\;10^{10}$.
The experience is repeated 100 times, which represents 3,500 data sets.

\begin{figure}[htbp!]
\begin{center}
\includegraphics[width=0.7\columnwidth]{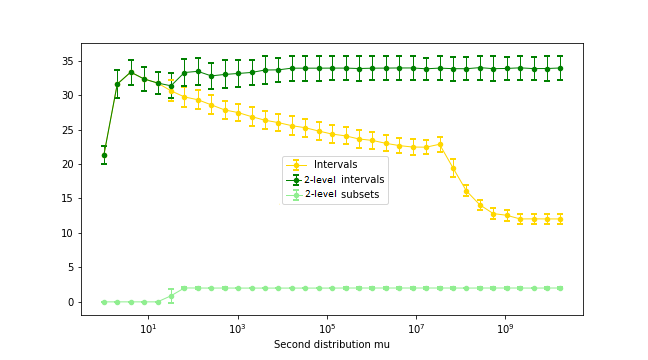}
\end{center}
\caption{Number of intervals obtained using or not the two-level method, for a mixture of two Gaussian distribution with far different ranges}
\label{fig:dynamicRange}
\end{figure}

Figure~\ref{fig:dynamicRange} reports the mean and standard deviation of the number intervals obtained using or not the two-level method, as well as the number of involved data subsets.
For $\mu_2=\mu1=1$, there is one single Gaussian distribution and both methods build around 21 intervals.
For $\mu_2 \in [2; 32]$, both methods build the same histogram to summarize the Gaussian mixture, using 31 to 33 intervals.
For $\mu_2 \geq 50$, the standard method suffers once again from the very large range of values in the data set. The two-level method splits the data set into two subsets, one per Gaussian component, and exploits around 34 intervals to summarize the underlying distribution.

\begin{figure}[htbp!]
\begin{center}
\includegraphics[width=0.4\columnwidth]{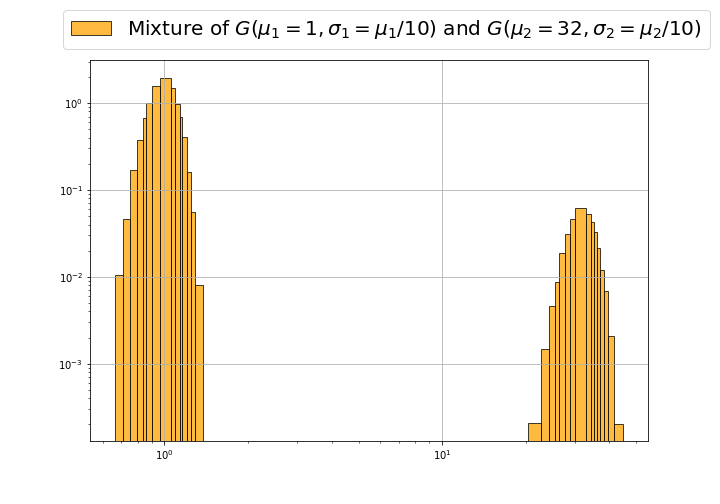}
\includegraphics[width=0.4\columnwidth]{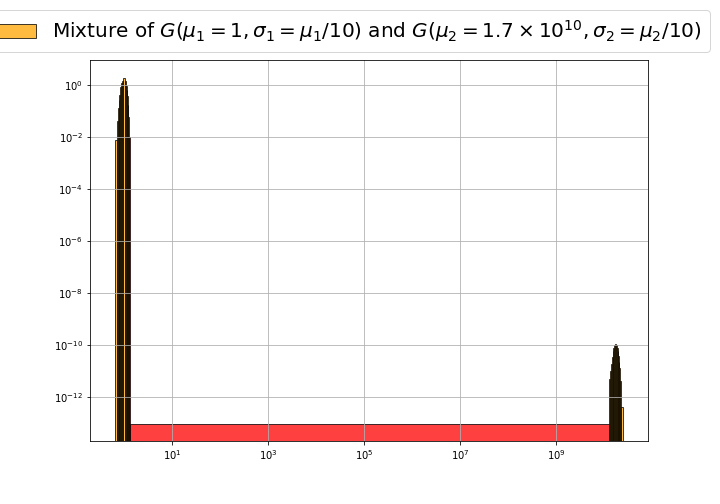}
\end{center}
\caption{Histograms obtained using the two-level method for the Gaussian mixture distribution, on the $\log \times \log$ scale}
\label{fig:gaussianMixtureHistogramsLogLog}
\end{figure}

\begin{figure}[htbp!]
\begin{center}
\includegraphics[width=0.4\columnwidth]{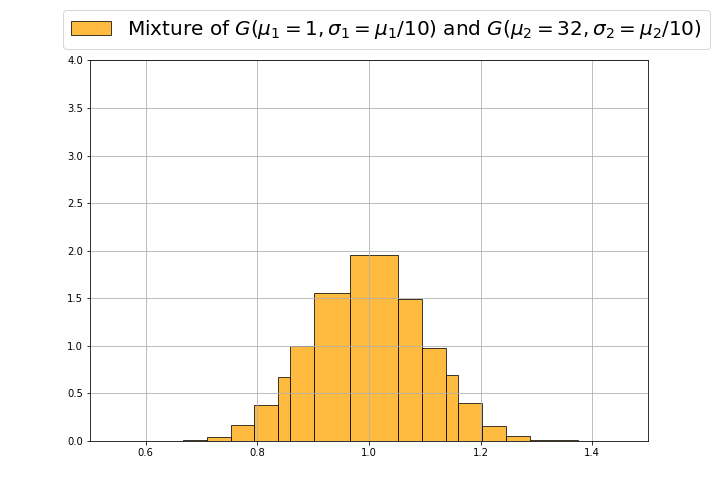}
\includegraphics[width=0.4\columnwidth]{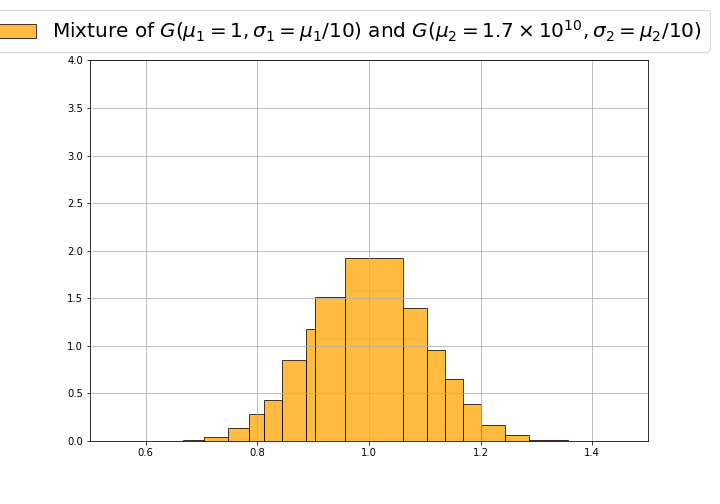}
\end{center}
\caption{Histograms obtained using the two-level method, for the Gaussian distribution $G(\mu=1, \sigma=0.1)$ and different distributions of outliers, with a focus on $X \in [0.5; 1.5]$}
\label{fig:gaussianMixtureHistograms}
\end{figure}

The histograms built using the two-level method are displayed for $\mu_2=32, 1.7 \times 10^{10}$ using a $\log \times \log$ scale in Figure~\ref{fig:gaussianMixtureHistogramsLogLog} and using the standard scale with a focus on the first Gaussian component in Figure~\ref{fig:gaussianMixtureHistograms}.
This shows that the Gaussian mixture distribution is correctly approximated for a very large range of values.

\subsection{Scalability}
\label{sec:scalability}

The objective of this experiment is to evaluate the scalability of the method in the case of a data set with a complex underlying distribution of values. 
We exploit a Gaussian mixture with 21 components where the mixture weights are distributed according to a Binomial distribution $B(n=20, p=0.5)$. We have $p(component=i)=\binom{20}{i} 2^{-20}$, with each mixture component based on a Gaussian distribution $G(\mu=i, \sigma=1/4)$.
We generate data sets from this distribution for size $n=2^i; 1 \leq i \leq 30$ ranging from 2 to one billion. The experiment is repeated only once for scalability reasons.

\paragraph{Accuracy of the histograms.}

\begin{figure}[htbp!]
\begin{center}
\includegraphics[width=0.48\columnwidth]{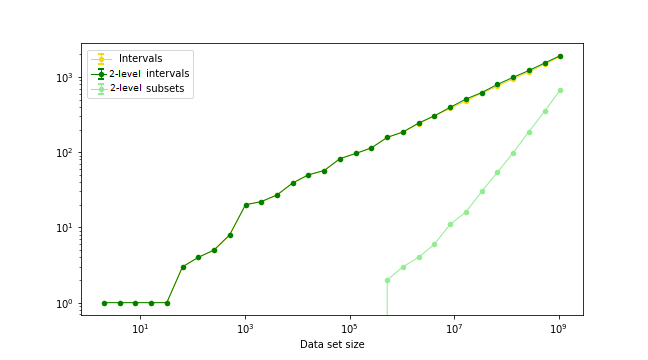}
\includegraphics[width=0.48\columnwidth]{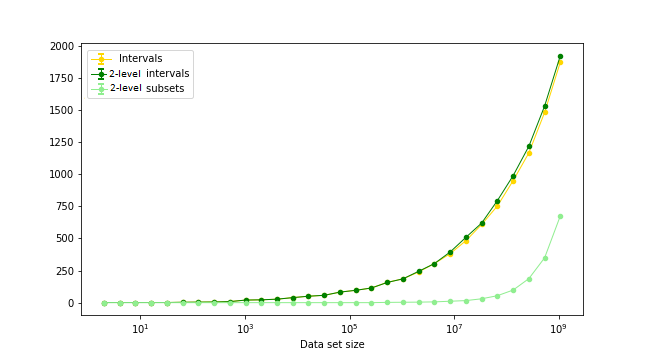}
\end{center}
\caption{Number of intervals obtained using or not the two-level method for large scale data sets, displayed using a log or standard scale}
\label{fig:scalabilityStudyIntervals}
\end{figure}

Figure~\ref{fig:scalabilityStudyIntervals} reports both on a standard and a log scale the mean and standard deviation of the number intervals obtained using or not the two-level method, as well as the number of involved data subsets.
The two-level method is triggered for data sets with size beyond half a million and the number of data sub sets then increases regularly until reaching around 700 for the largest data set of size one billion.
The number of intervals in the histogram increases approximately as the cubic root of the size of the data set. For example, about 100 intervals are built for $n=2^{17}\approx 1.3 \times 10^5$, and about 1000 intervals for $n=2^{27}\approx 1.3 \times 10^8$.
Both the standard and two-level methods build comparable numbers of intervals, as shown in Figure~\ref{fig:scalabilityStudyIntervals} on the standard scale display. The two-level method builds slightly more intervals for large data sets. Indeed, whereas the standard method is fully regularized on the whole data set, the two-level method exploit the G-Enum method independently per sub data set, resulting in a locally regularized approach.

\begin{table}[!htbp]
\begin{small}\begin{center} \begin{tabular}{cccc}
\includegraphics[width=0.24\columnwidth]{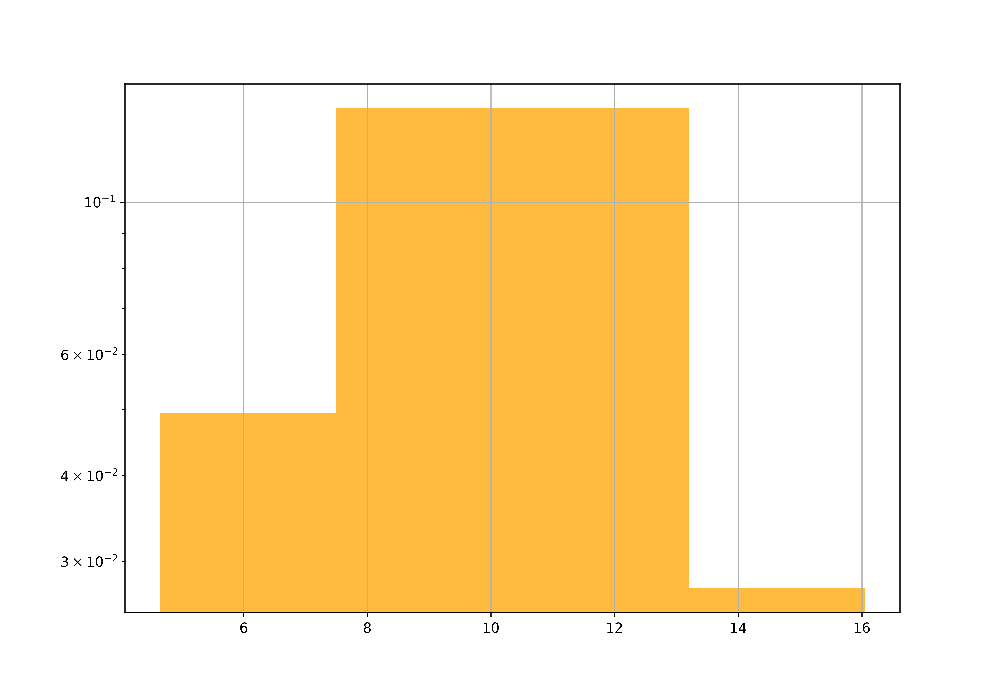}
&
\includegraphics[width=0.24\columnwidth]{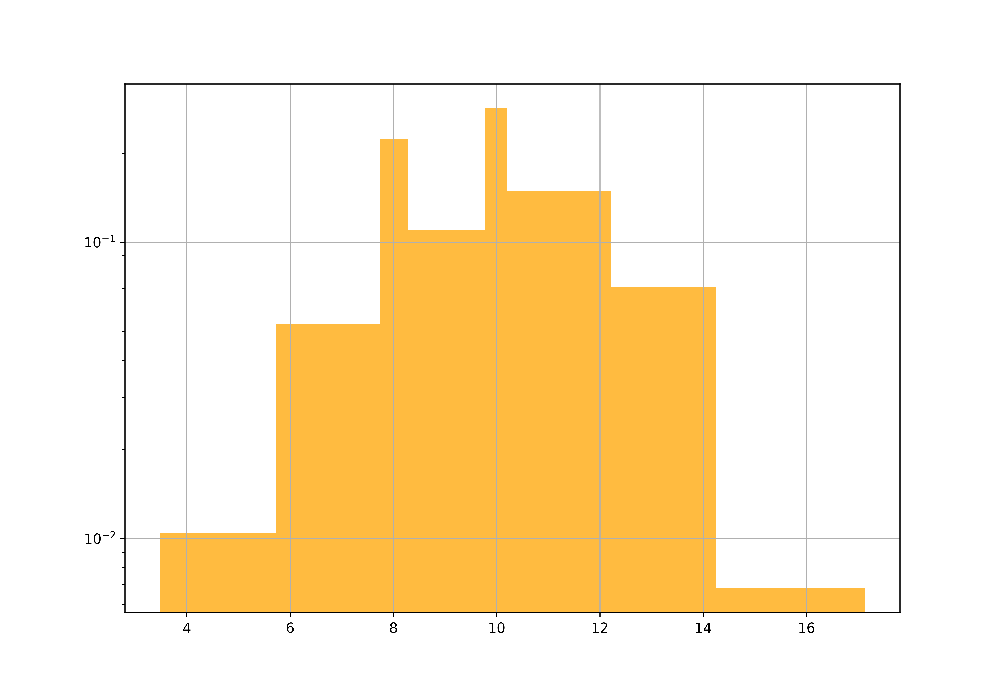}
&
\includegraphics[width=0.24\columnwidth]{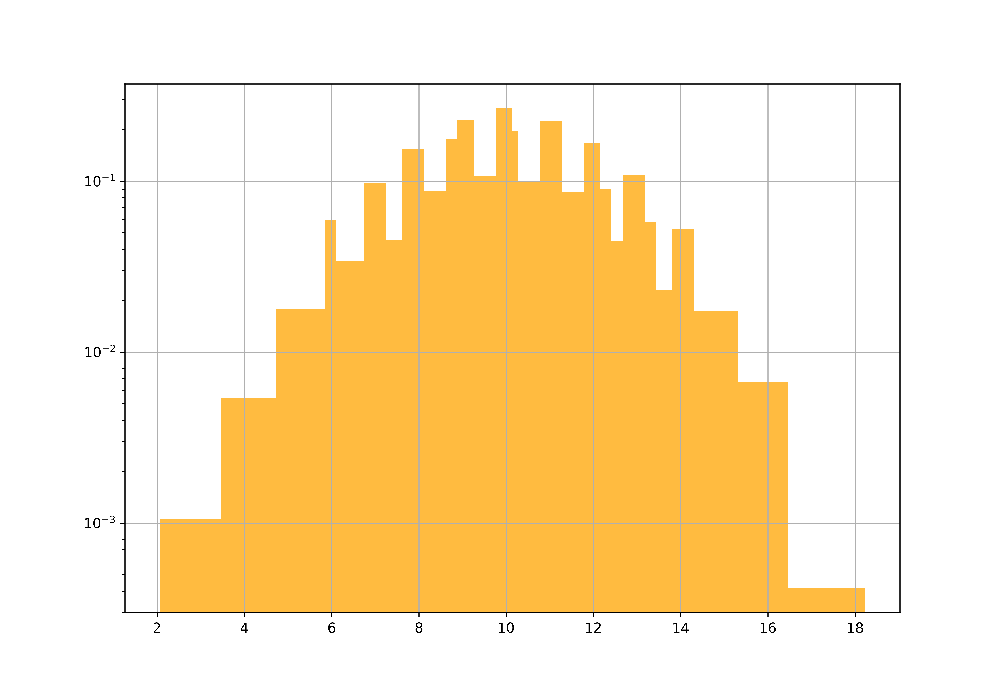}
&
\includegraphics[width=0.24\columnwidth]{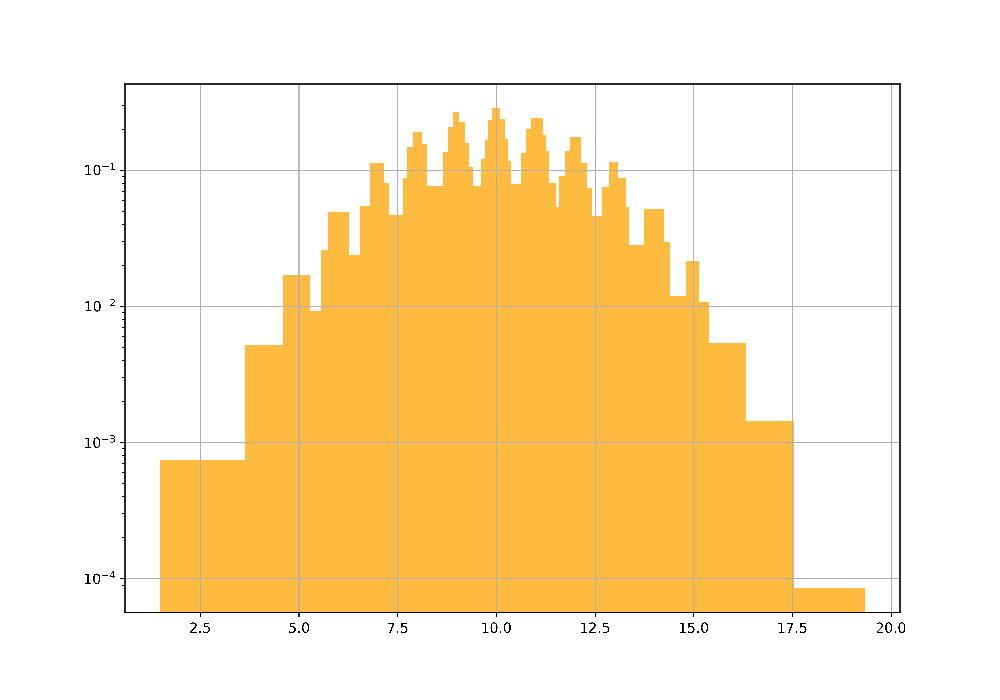}
\\
$n=6.4 \times 10^1 $ & $n=5.1 \times 10^2$ & $n = 4.1 \times 10^3$ & $n=3.3 \times 10^4 $ \\
\includegraphics[width=0.24\columnwidth]{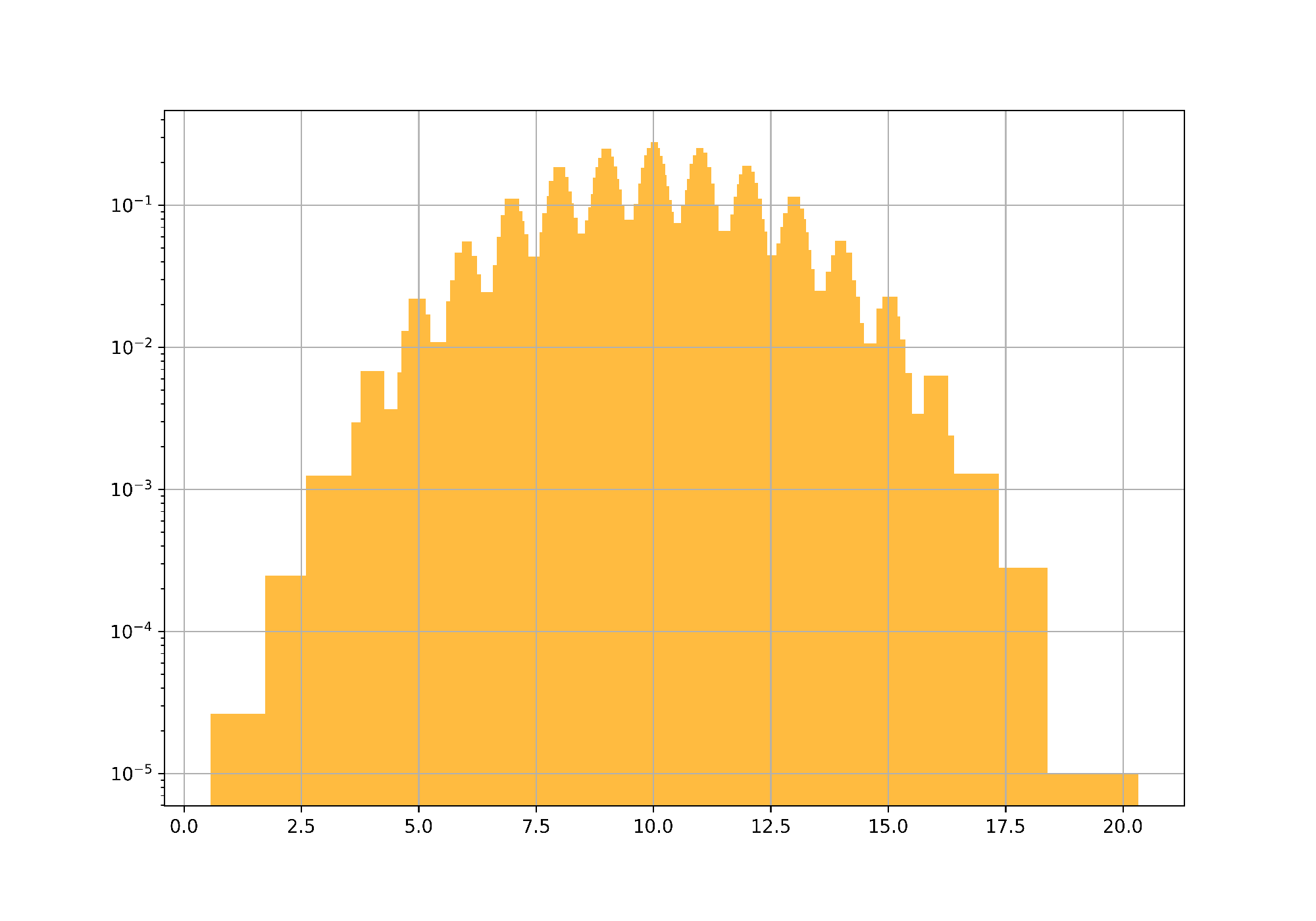}
&
\includegraphics[width=0.24\columnwidth]{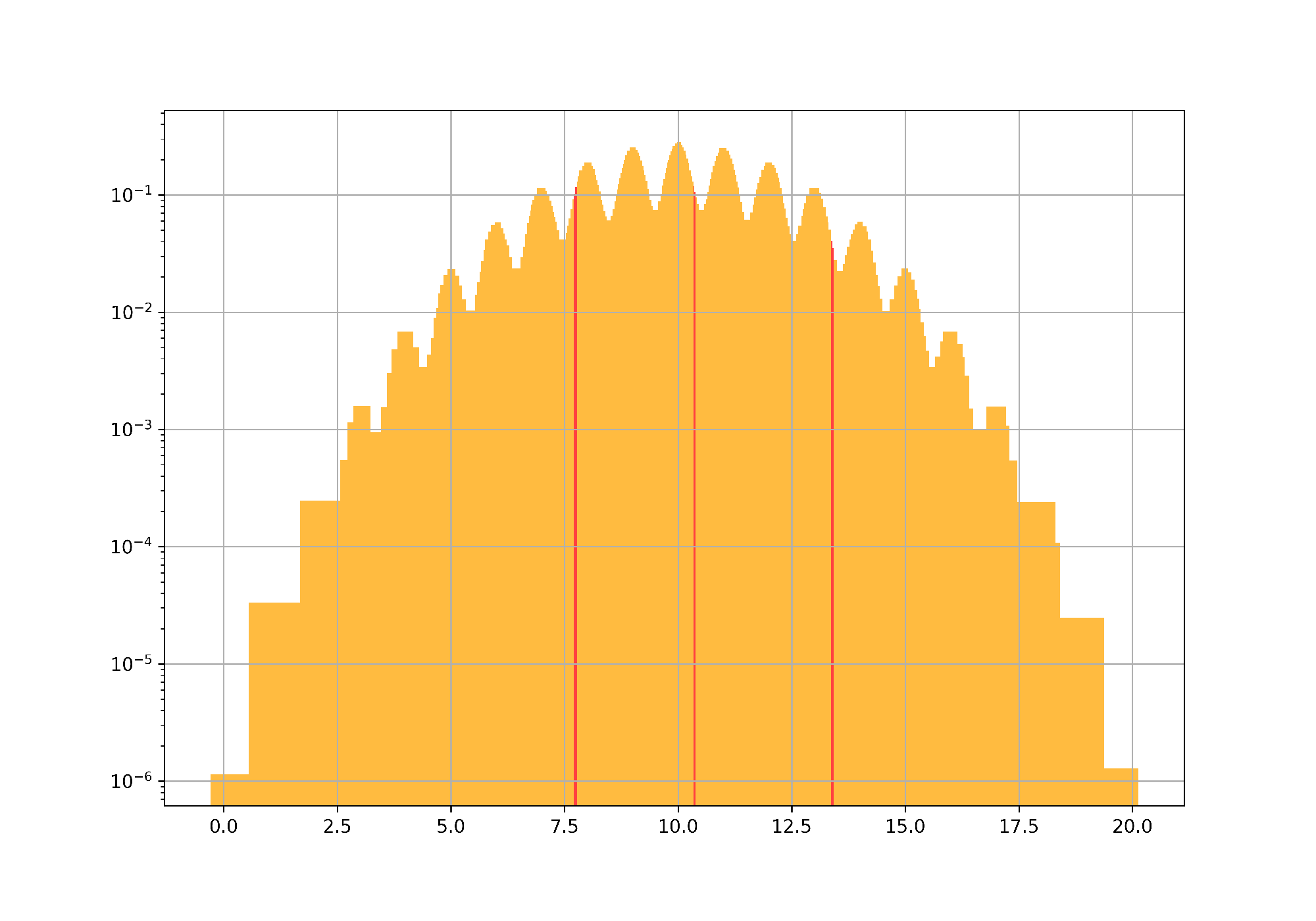}
&
\includegraphics[width=0.24\columnwidth]{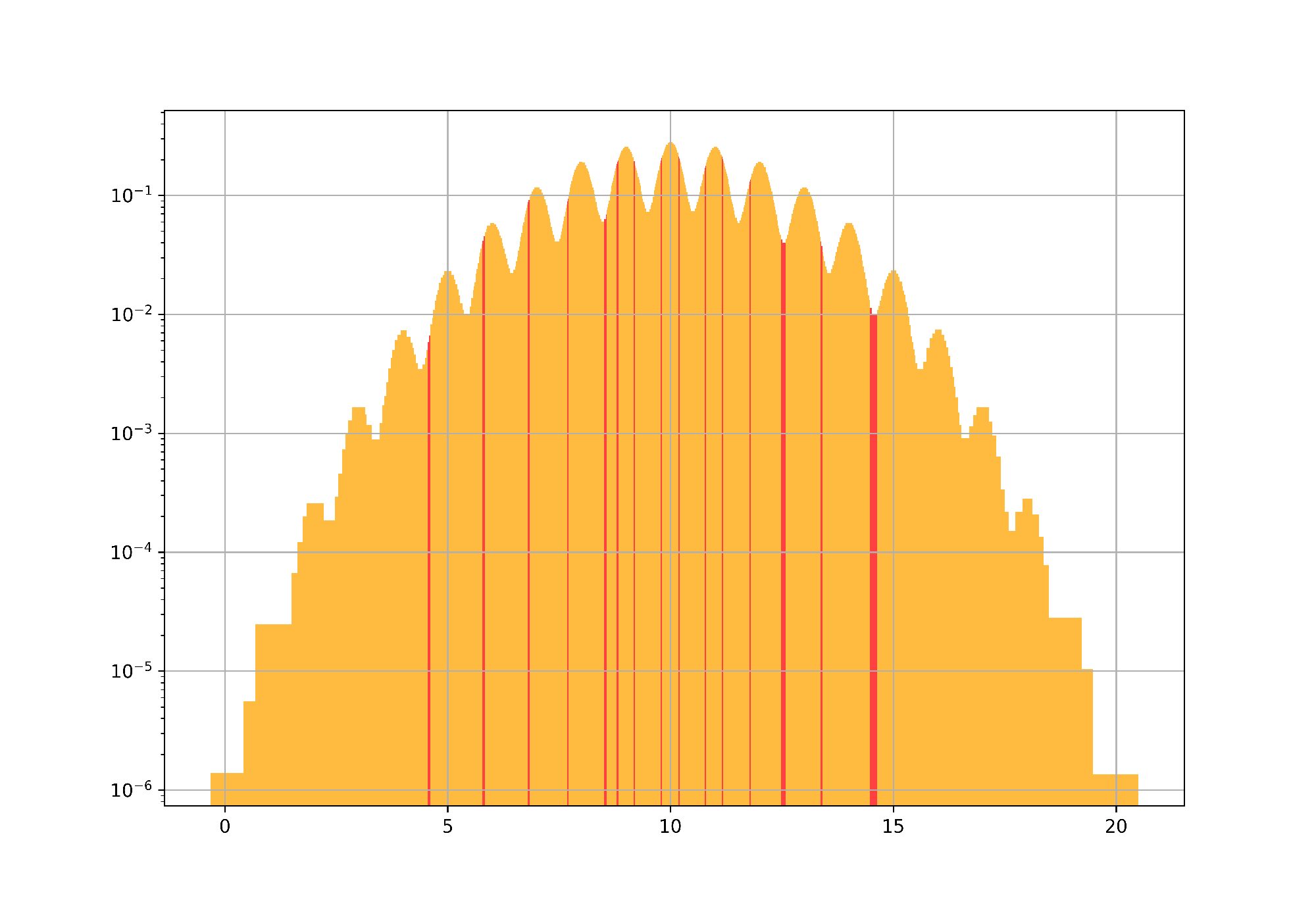}
&
\includegraphics[width=0.24\columnwidth]{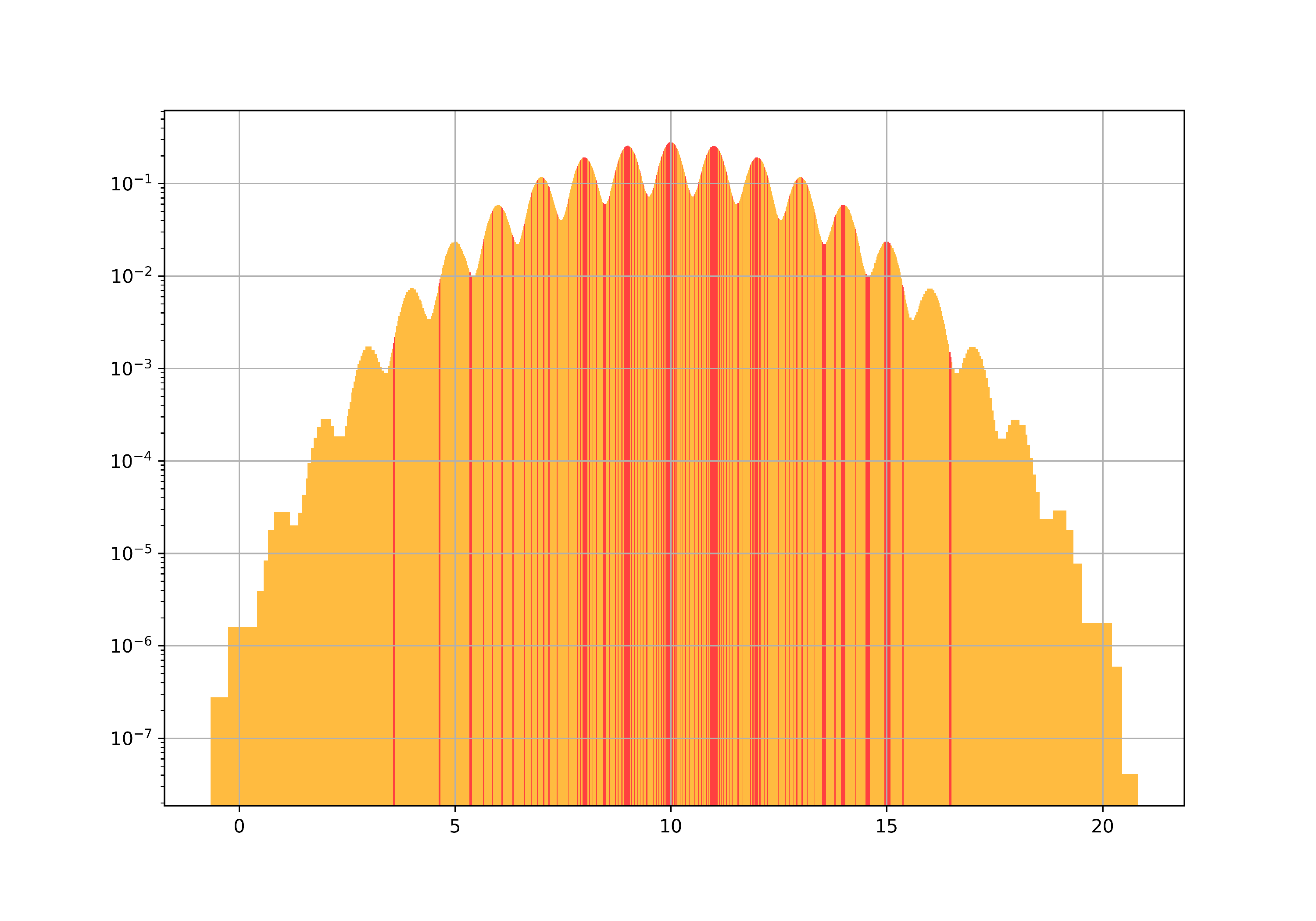}
\\
$n=2.6 \times 10^5$ & $n = 2.1 \times 10^6$ & $n=1.7 \times 10^7 $ & $n=1.3 \times 10^8$
\end{tabular}\end{center}\end{small}
\caption{Histograms built for data sets of increasing size}
\label{tab:scalabilitySamples}
\end{table}

\begin{figure}[htbp!]
\begin{center}
\includegraphics[trim={0 0.5cm 0 1.5cm 0},width=0.99\columnwidth]{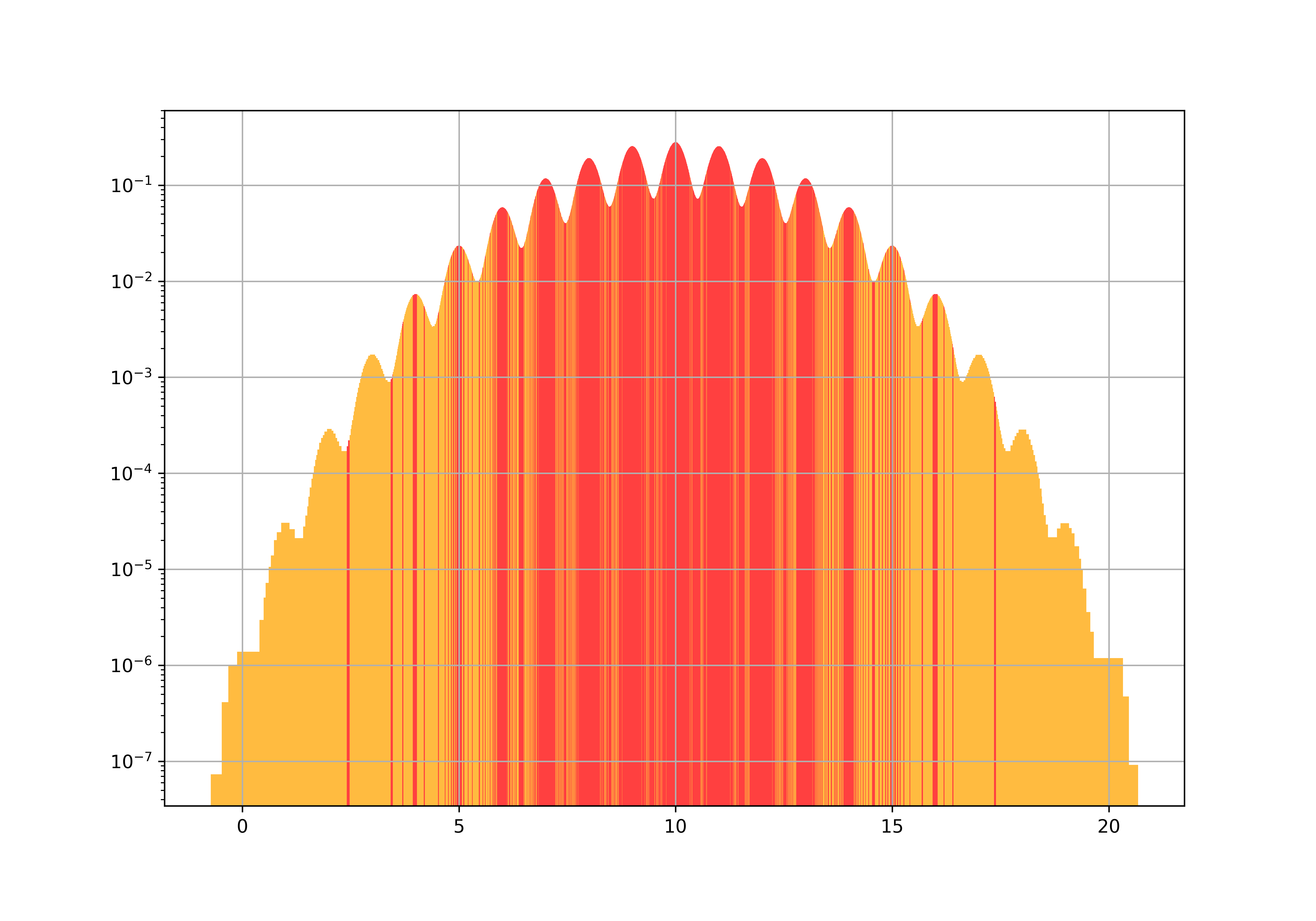}\\
\includegraphics[trim={0 2.5cm 0 1.5cm 0},width=0.32\columnwidth]{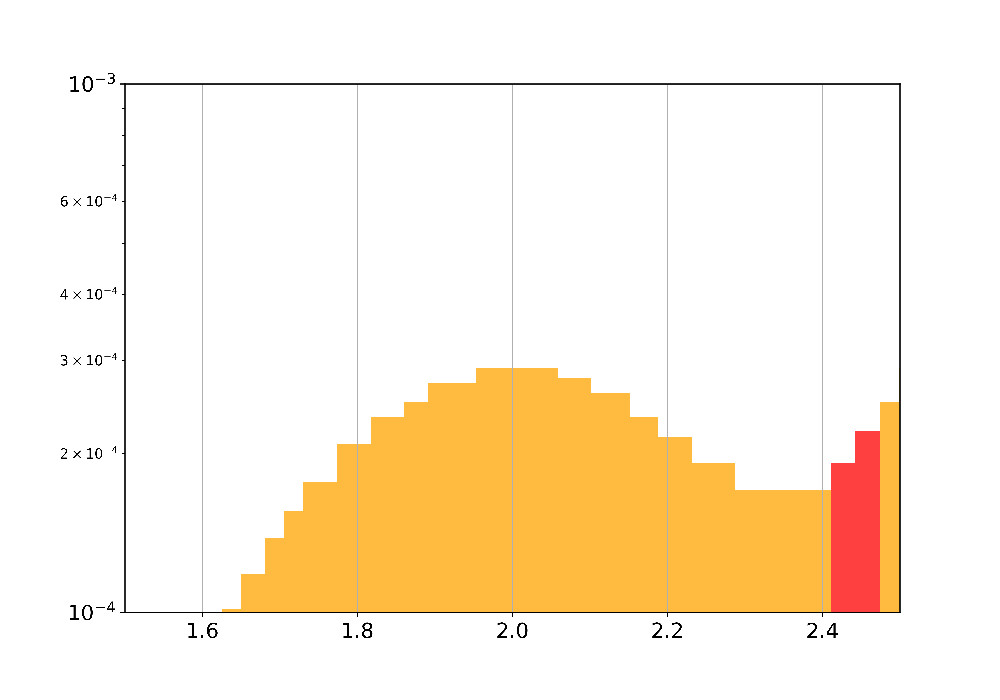}
\includegraphics[trim={0 2.5cm 0 1.5cm 0},width=0.32\columnwidth]{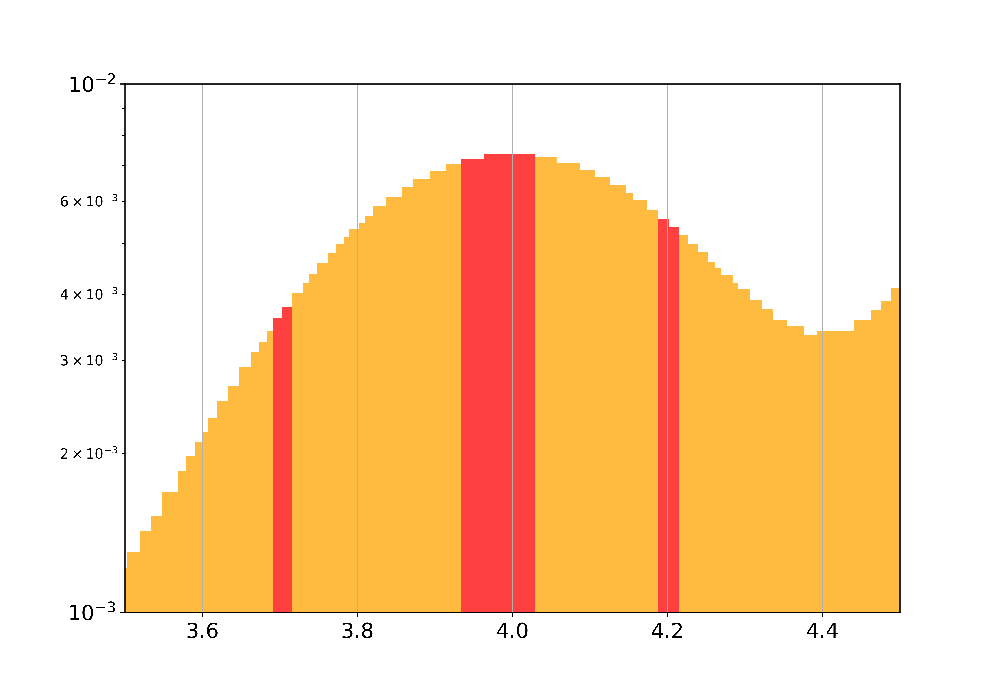}
\includegraphics[trim={0 2.5cm 0 1.5cm 0},width=0.32\columnwidth]{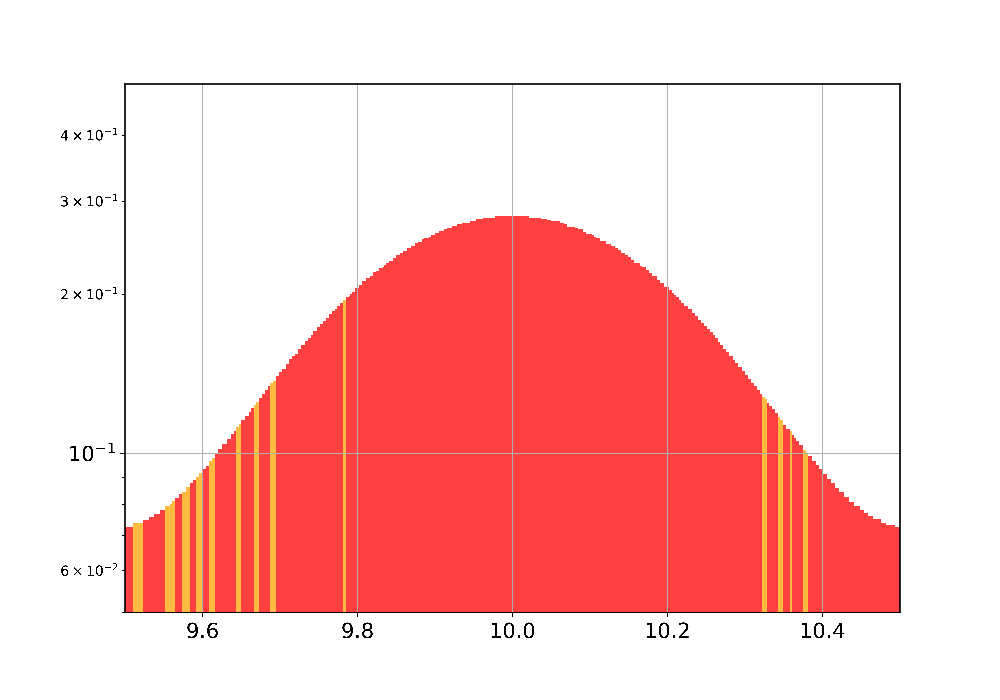}
\end{center}
\caption{Histograms built for a large data set with one billion data entries, with a zoom on the components 2, 4 and 10 of the Gaussian mixture}
\label{fig:scalabilityLargestSample}
\end{figure}

Table~\ref{tab:scalabilitySamples} displays the histograms built for a series of data sets with increasing sizes on six orders of magnitude, from $n=6.4 \times 10^1$ to $n=1.3 \times 10^8$.
The larger the data set, the more accurate the obtained histogram.
With few data, only part of the distribution is discovered, and the histograms are blind to the tails of the distribution and to most of its patterns.
As the amount of processed data increases, the distribution is summarized more and more completely and accurately.

The most detailed histogram obtained with one billion of data entries is displayed in Figure~\ref{fig:scalabilityLargestSample}. It consists in about 1900 intervals with heavily unbalanced distribution of lengths, frequencies and densities, ranging from $0.0003$ to $0.7$ for the lengths, from 20 to $14,500,000$ for the frequencies and from $7.3 \; 10^{-7}$ to 0.3 for the densities.
Figure~\ref{fig:scalabilityLargestSample} also shows a zoom of the histogram on the components 2, 4 and 10 of the underlying Gaussian mixture distribution. The Gaussian $component_{10}$ is by far the most populated and its piecewise constant density estimation provided by the histogram is both very smooth and accurate, using 203 intervals in $[9.5; 10.5]$. 
According to the figures in Table~\ref{tab:scalabilitySamples}, the Gaussian $component_2$ was not even sampled for $n \leq 10^4$ and and its shape began to roughly appear for $n \geq10^7$. With $n \geq 10^9$, this Gaussian $component_2$ is pretty well approximated in Figure~\ref{fig:scalabilityLargestSample} using 25 intervals in $[1.5; 2.5]$, although the quality of the approximation is far from that of $component_{10}$.
Note that even with one billion data entries, the first Gaussian $component_{0}$ is still roughly approximated, using only 7 intervals in $]\infty; 0.5]$

\paragraph{Computation time.}

\begin{figure}[htbp!]
\begin{center}
\includegraphics[width=0.7\columnwidth]{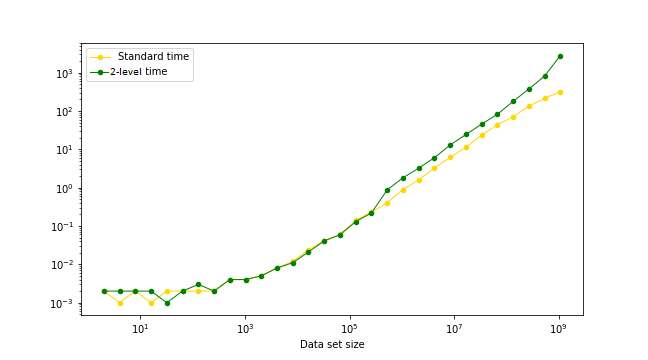}
\end{center}
\caption{Computation time in seconds using or not the two-level method for large scale data sets}
\label{fig:scalabilityStudyTimes}
\end{figure}

The experiments are performed on a PC under Windows Server 2012, with a processor Intel Xeon Gold 3150 2.7 GHz and 192 GB RAM, using a single core as the implementation is not parallel.
Figure~\ref{fig:scalabilityStudyTimes} reports the computation time in seconds for the standard and two-level methods. In order to focus on the computation time of each method, the initialization time that is common to both methods is not taken into account. This initialization time mainly consists in reading the data from an input file and initializing an input contingency table in memory with the pairs (value, frequency) sorted by values.
For data sets with size below half a million, the PICH criterion is not triggered and the two-level method reduces to the standard method. 
For larger data sets, the PICH criterion is triggered and the two-level method requires between two and three time more computation time than the standard method, as expected.
For the largest data set with one billion data entries, the histogram required a few hours and around 150 GB RAM to be built.

\paragraph{Scalability and heavy tail distribution.}
The previous experiment has shown that the histograms built using or not the two-level method have similar quality based on about the same numbers of intervals.
We perform a last challenging experiment that combines scalability and a heavy tail distribution.
We exploit a Gaussian mixture with 21 components where the mixture weights are distributed according to a Binomial distribution $B(n=20, p=0.5)$, with each mixture component based on a Gaussian distribution $G(\mu=10^i, \sigma=\mu/4)$. The range of the means of the Gausssian components is $[1; 10^{20}]$, instead of $[0;20]$ in the preceding scalability experiment.

\begin{figure}[htbp!]
\begin{center}
\includegraphics[trim={0 0.5cm 0 1.5cm 0},width=0.99\columnwidth]{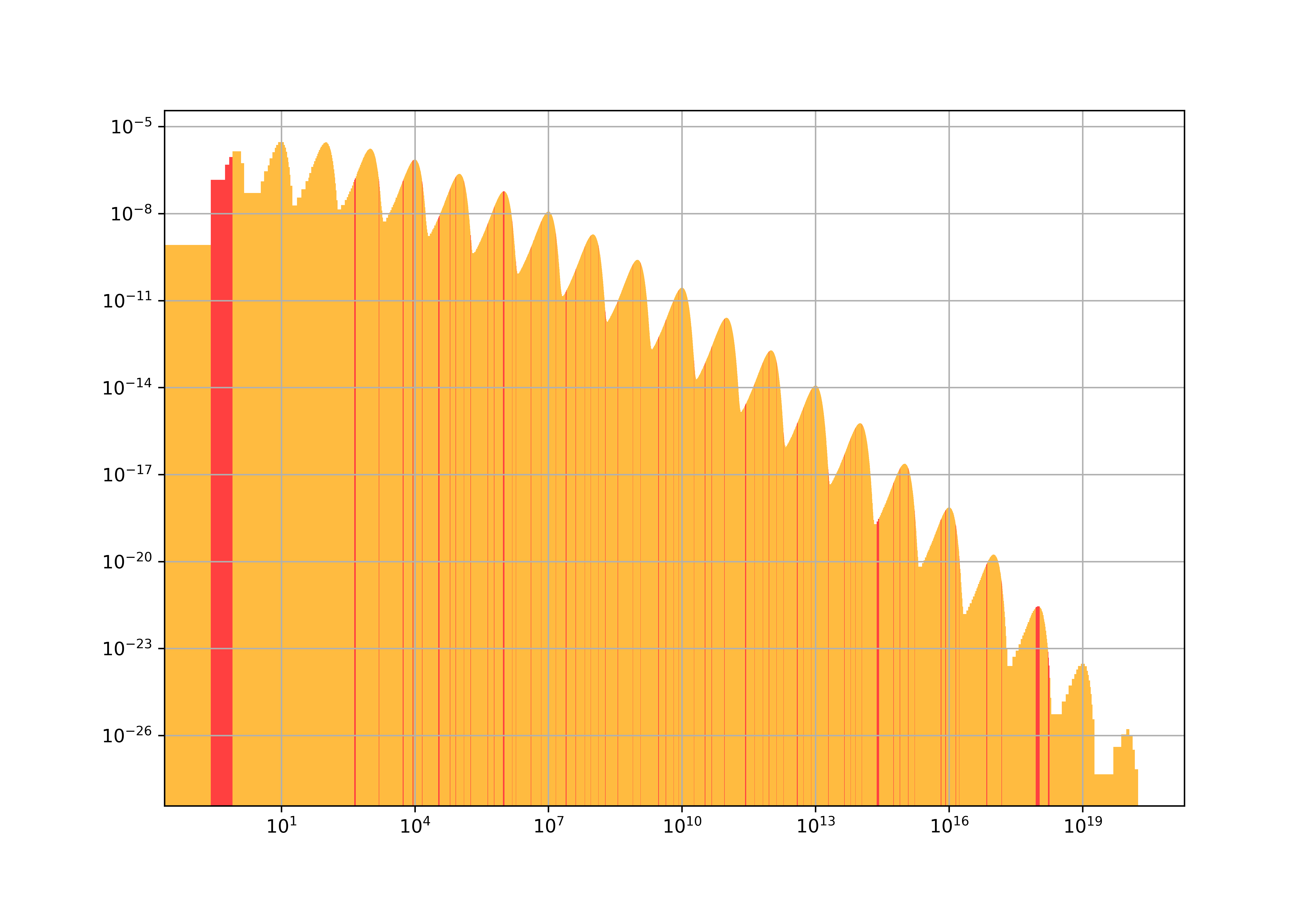}
\end{center}
\caption{Histograms built for a large data set with one billion data entries and a heavy tail distribution}
\label{fig:scalabilityLargestDynamicRangeSample}
\end{figure}

Without the two-level method, the width of the $\epsilon$-bin is about $10^{20}/10^9=10^{11}$, so that the first eleven mixture components are merged into the first interval of the built histogram. Using the two-level method, the histogram is very detailed and all the mixture components are approximated using altogether about 3800 intervals.
The most detailed histogram obtained with one billion of data entries is displayed in Figure~\ref{fig:scalabilityLargestDynamicRangeSample}.

\section{Future work}
\label{sec:futureWork}
In this section, we discuss future work.

\subsection{Convergence to data distribution}

The experiments on a large scale data set reported in Section~\ref{sec:scalability} suggest that histograms built by the proposed method seem to converge towards the underlying distribution as the size of the data set increases.
In future work, it would be interesting to investigate on this property from a theoretical point of view.

Many alternative histogram approaches have been studied in the literature, with deep theoretical insights w.r.t. their convergence properties. Most of these approaches rely on strong assumptions such as the existence of a continuous density, of first and second derivatives with sometimes bounded derivatives. Some approaches assume that the number of intervals and the frequency per interval increase as $n \rightarrow \infty$ to provide convergence properties. The quality of the histograms is assessed using statistical distances such  as the Kullback-Leibler divergence, the Hellinger distance or the mean square error.
Contrary to these methods, and following the minimum description length (MDL) approach of Rissanen exploited in the K\&M histogram method, we make no assumption regarding the data distribution and only focus on the data compression that is possible if there are patterns in the data. 
Studying the convergence property of this kind of approach is an open problem that might be hard to tackle.

\subsection{Visualization tools}

Using the G-Enum method, each histogram model $M$ is evaluated using $cost(M)$, that is its coding length according to the MDL approach.
This criterion can be normalized using the cost of the null histogram model $M_\emptyset$ that contains one single interval, according to
$$level(M) = 1 - \frac{cost(M)}{cost(M_\emptyset)}.$$
The criterion $Level$ that represents a compression rate and assesses the quality of a histogram is useful in practice as it allows to compare the interest of several variables through their histogram and allows the data analyst to focus on the most interesting variables.
However, this criterion is no longer available when the two-level method is triggered. As we have a global histogram obtained from the aggregation of the sub-histograms, we could compute a global $Level$, as if the global histogram had been obtained with the G-Enum method alone using the smallest $\epsilon$-bin among all the sub histograms. This approach might not be fully satisfying, as the resulting $Level$ is likely to drop to 0 when the sub histograms are built from data subsets with radically different ranges of values. This needs to be further investigated in future work.

\medskip
Other useful improvements for data exploration involve dedicated visualization tools, which could easily switch between the standard and logarithmic scale, either for the values ($X$ axis) using the $\log_{\mathcal{D}}^{(cr)}$ function or for the probabilities or densities ($Y$ axis) using the standard $\log$ function.  Indeed, the $\log_{\mathcal{D}}^{(cr)}$ function introduced in Section~\ref{sec:logTransformation} looks convenient to visualize any data, negative, null or positive using a logarithmic scale
It should be noticed that the horizontal upper line of the histogram bars built in the initial value domain should be drawn with a logarithmic slope when represented using the the logarithmic scale.
Zooming features may be convenient in the case of data with  a heavy tail distribution.
A visualization tool should also be able to automatically propose a default view (choosing what to visualize on each axis, using either the standard or logarithmic scale), as some distributions are so unbalanced that nothing can be seen on some views.
Finally, it could be interesting to keep some basic statistics per histogram interval, such as the mean, standard deviation, minimum and maximum of the values within the interval. This could be helpful in a data exploration context to inspect the distribution tails and to help identifying the outliers.

\subsection{Faster implementation for better applicability}

Overall, the G-Enum algorithm and the two-level method have a theoretical time complexity of O$(n \log n)$, which makes them suitable to process large data sets. 
Let us first remind the solutions implemented in these algorithms to push the limits of their applicability as far as possible:
\begin{itemize}
	\item the greedy bottom-up heuristic in O($n \log n)$ in the G-Enum algorithm is used rather than the optimal algorithm in O$(n^3)$ based on dynamic programing,
	\item the G-Enum method allows an automatic choice of the best histogram granularity $G$ without requiring a user parameter,
	\begin{itemize}
		\item the $E=10^9$ internal parameter, which corresponds to the maximum granularity, has been chosen to be as large as possible within the limits of computer numerical precision,
		\item the optimization of the granularity relies on a loop on granularities increasing by powers of two to keep a O$(n \log n)$ time complexity,
	\end{itemize}
	\item the two-level method method has been suggested to deal with data sets with outliers or heavy tail distribution,
\begin{itemize}
	\item it accounts for the limits of floating-point representation, using the $\log(\mathbb{R}^{(cr)})$ transformation of the data,
	\item it exploits a PICH criterion to split a data set into well-conditioned data subsets
	\item it keeps an overall time complexity of O$(n \log n)$.
\end{itemize}
\end{itemize}

Decreasing the time complexity below O$(n \log n)$ might not be feasible without important loss of accuracy, as each data must be seen at least once, which requires O$(n)$ time.
Still, the time complexity comes with a constant factor $\alpha$ such that the computation time can be bounded by $\alpha \times  n \log n$. 
Decreasing $\alpha$ by a percentage may not be worth it, but decreasing it by a factor my be useful in practice. We suggest below several possible solutions that may altogether result in decreasing the computation time by a factor of 2 to 10, depending on the data to process.

\paragraph{Early stopping for the G-Enum heuristic.}
The time complexity of the G-Enum heuristic is O$((2 + \log_2 E - \log_2 n)n \log n)$ as a function of both $n$ and $E$.
Instead of evaluating all the granularities $G_i=2^i, 1 \leq i \leq 30, G_{30} \approx E$, the G-Enum heuristic could stop as soon as the resulting model cost decreases.
Indeed, as the G-Enum criterion is regularized, the model costs tend to decrease at the beginning of the loop when finer granularities allow more accurate histograms with better likelihood and they tend to increase at the end of the loop when too fined granularities are penalized by larger prior terms. This behavior is frequently observed in practice and the optimal models in case of well conditioned data are often found for $G_i \leq n$. For  example, with a data set of size $n=10,000$, the factor $(2 + \log_2 E - \log_2 n) \approx 19$ could drop down to $2$, with a computation time almost ten times faster.
To get a better trade-off between computation time and model accuracy, we evaluated the following stopping criterion: stop exploring the granularities as soon as $G_i > \sqrt{n}$ and at least 3 successive granularities do not improve the model cost. This new trade-off looks promising, as computation time could be largely reduced without any loss in model accuracy. 

\paragraph{Alternative trade-offs in the two-level method.}
Altogether, the two-level method relies on the maximum granularity parameter $E=10^9$ and on the PICH criterion which exploits a granularity threshold $t_E=\sqrt E \log E$ and a colliding frequency threshold $t_c = \log n$. 
This has been discussed in Section~\ref{sec:wchProperty} and further investigated in Section~
\ref{sec:thresoldICH}, resulting in acceptable trade-offs between the competing criterions of automation, theoretical optimality, accuracy and scalability.
For applicative contexts where the scalability criterion is the main issue, new Pareto optimal trade-offs could be used for the choices of $E, t_E$ and $t_c$ to comply with specific applicative constraints.

\paragraph{Faster split of the data set in the two-level method.}
The purpose of the first level of the two-level method is to split a PICH data set into a list a PWCH data subsets. This first level does not require optimized results of equivalent quality as those used to build histograms in the second level. If scalability is an issue, this split heuristic could be simplified or even replaced by alternative more time efficient heuristics. 

\paragraph{Parallelisation of the algorithms.}
At least the second level of the two-level method looks easy to parallelize on the basis of each data subset. Further work is still necessary to efficiently parallelize the whole heuristic.

\subsection{Big data and fast data}

The scalability experiments in Section~\ref{sec:scalability} show that more accurate models are obtained with more data. Although the experiments were performed using an artificial data set, real world data sets often come with a long tail distribution that requires a lot a data to be accurately approximated.
In many use cases, such accurate summaries could be used to query huge data sets stored in big data infrastructures or even to keep track of fast data streams and still exploit them when the data is not longer available.
Note that the obtained histogram summaries are very parsimonious, with for example about 1000 intervals instead of $n=1.3 \times 10^8$ data entries in the scalability experiments. 
Exact results with a precision better than one billionth are unnecessary in most contexts of data exploration. Using accurate histogram summaries could then be a time, memory and energy efficient alternative to solutions such as \emph{elastic search}, which requires vast amounts of processing time and storage capacities. 

The two-level method is scalable enough to process data sets of size up to one billion of data entries within a few hours and around 150 GB of memory.
This is not sufficient in the context of big data, where this gigabyte scale algorithm needs to be extended to process terabytes or even petabytes of data.
Some divide and conquer principles need to be exploited to scale up the two-level method.
The context of fast data streams is still more challenging as the computational resources are likely to be more tightly bounded and the data can be seen only once.
Furthermore, in the case of non stationary data distribution, stream mining algorithms have to cope with the time evolution of distributions.


\subsection{Pushing the limits further and beyond}

Hello happy reader, you are in the Easter egg section.
Overall, the proposed approach relies on the G-Enum method that have strong theoretical foundations.
Still, in the end, numerical methods have to be applied on data sets with computer real values in $\mathbb{R}^{(cr)}$, that do not always behave as the mathematical real values from $\mathbb{R}$.
The G-Enum method has thus been extended to deal with outliers and with numerical precision limits.
Altogether, we call this histogram method \emph{WAOH}, as Widely Applicable Optimal Histograms.

As building adversarial data sets is fairly easy, the WAOH method may fail in numerous cases. Future work is necessary to push the limits of the method further and beyond.
However, these improvements may not be worth it.
Indeed, the data sets that could benefit from such improvements are likely to be outliers in the distribution of all real world data sets (see Theorem~\ref{thEaster}).

\begin{theorem}
\label{thEaster}
Let $\mathds{D}^{(cr)}$ be the set of all possible data sets with data entries in $\mathbb{R}^{(cr)}$. Let $\mathds{D}_{rw, t}^{(cr)}$ be the set of all real world data sets produced by man kind until time t and $\mathds{D}_{rw_u, t}^{(cr)}$ the related subset of useful real world data sets. We have
\begin{equation*}
\forall t \in \mathbb{R}, \frac{|\mathds{D}_{rw, t}^{(cr)}|}{|\mathds{D}^{(cr)}|} \approx 0 \quad \mathnormal{and} \quad
lim_{t\to\infty}\frac{|\mathds{D}_{rw_u, t}^{(cr)}|}{|\mathds{D}_{rw, t}^{(cr)}|} =0.
\end{equation*}
\end{theorem}
\begin{proof}
The proof could not be included in the paper for latex compile error reasons.
\end{proof}

\section{Conclusion}
\label{sec:conclusion}

This paper starts from the G-Enum histogram method that have strong theoretical foundations.
Still, in the end, numerical methods have to be applied on data sets with values stored on computers using a floating-point representation, that do not always behave as the mathematical real values from $\mathbb{R}$.
The G-Enum method has then been embedded into the two-level method to finely account for the limits and pitfalls of floating-point representation.
This heuristic allows to extend the applicability of the method to a wide range of data sets, including the case of outliers or heavy tail distribution.
Extensive experiments demonstrate the benefits of the approach, that allows to build accurate histogram summaries of data sets within efficient computation time.
Future works include extensions to the processing of huge data stores or fast data streams, which could bring time, memory and energy efficient building brick for many data exploratory or supervised data mining tasks.

\bibliographystyle{apalike} 
\bibliography{TwoLevelHistograms}

\end{document}